\documentclass{article} % For LaTeX2e
\usepackage{iclr2026_conference,times}

\usepackage{amsmath,amsfonts,bm}

\def\eqref#1{equation~\ref{#1}}
\def\1{\bm{1}}

\DeclareMathAlphabet{\mathsfit}{\encodingdefault}{\sfdefault}{m}{sl}
\SetMathAlphabet{\mathsfit}{bold}{\encodingdefault}{\sfdefault}{bx}{n}

\usepackage{hyperref}
\usepackage{url}
\usepackage{tabularx}
\usepackage{graphicx}
\usepackage{enumitem}
\usepackage{algorithm}
\usepackage{algorithmicx} 
\usepackage[noend]{algpseudocode}
\usepackage{amsmath}  
\usepackage{amsthm}
\usepackage{multirow}
\usepackage{bm}
\usepackage{mathrsfs}
\usepackage{placeins}

\usepackage{subfigure}

\usepackage[utf8]{inputenc} % allow utf-8 input
\usepackage[T1]{fontenc}    % use 8-bit T1 fonts
\usepackage{hyperref}       % hyperlinks
\usepackage{booktabs}       % professional-quality tables
\usepackage{amsfonts}       % blackboard math symbols
\usepackage{nicefrac}       % compact symbols for 1/2, etc.
\usepackage{microtype}      % microtypography
\usepackage{xcolor}         % colors

\usepackage{tabularray}
\usepackage{array}

\usepackage{multirow}
\usepackage[table,dvipsnames]{xcolor}

\newtheorem{theorem}{Theorem}
\newtheorem{lemma}{Lemma}
\newtheorem{definition}{Definition} 
\newtheorem{example}{Example}

\title{
\textit{HAMMER}: Hamiltonian Curiosity Augmented Large Language Model Reinforcement
}
\iclrfinalcopy
\author{
\textbf{Ming Yang}$^{1,2}$\thanks{Equal contribution. 
Ming Yang: idea, code, experiments, and writing. 
Xiaofan Li: experiment setup and writing. 
Work done during internship at Ant Group.} \quad
\textbf{Xiaofan Li}$^{3,2*}$ \quad
\textbf{Zhiyuan Ma}$^{2}$ \quad
\textbf{Dengliang Shi}$^{2}$ \\
\textbf{Jintao Du}$^{2}$ \quad
\textbf{Yu Cheng}$^{2}$ \quad
\textbf{Weiguo Zheng}$^{1}$\thanks{Corresponding author.} \\
\\
$^1$Fudan University\quad
$^2$Tiansuan Lab, Ant Group Co., Ltd.\quad
$^3$East China Normal University\\
\texttt{\small yangm24@m.edu.cn, funzi@stu.ecnu.edu.cn} \\
\texttt{\small \{mazhiyuan.mzy,dengliang.sdl,lingke.djt,cy122623\}@antgroup.com} \\
\texttt{\small zhengweiguo@fudan.edu.cn}
}

\begin{document}

\maketitle

% \begin{abstract}
% Recent curriculum reinforcement learning for large language models (LLMs) typically rely on difficulty-based annotations for data filtering and ordering. However, such methods suffer from costly \textit{pass@k} inference or manual labeling, and human-perceived difficulty level does not necessarily align with the model’s learning efficiency. 
% We propose a novel schema, namely 
% \textit{Hamiltonian curiosity augmented large language model reinforcement (HAMMER)},
% that transfers diversity metrics, commonly used in dataset evaluation, into the dynamic reinforcement learning procedure, where training samples are ordered via a minimum-semantic Hamiltonian path. 
% From a theoretical perspective of generalization bounds, diversity-driven ordering facilitates stable convergence.
% Empirical evaluations indicate that \textit{HAMMER} stimulates model ``curiosity'' and consistently achieves a 3\% to 4\% average accuracy gain across diverse inference benchmark datasets.
% \end{abstract}

\vspace{-4mm}
\begin{abstract}
\vspace{-1mm}
Recent curriculum reinforcement learning for large language models (LLMs) typically rely on difficulty-based annotations for data filtering and ordering. 
% \shancun{
However, such methods suffer from local optimization, where continual training on simple samples in the early steps can cause the policy to lose its exploration.
% } 
We propose a novel schema, namely 
\textit{Hamiltonian curiosity augmented large language model reinforcement (HAMMER)},
that transfers diversity metrics, commonly used in dataset evaluation, into the dynamic reinforcement learning procedure, where training samples are ordered via a minimum-semantic Hamiltonian path 
% \shancun{
making the initial training retrain more exploration.
% } 
From a theoretical perspective of generalization bounds, diversity-driven ordering facilitates stable convergence.
Empirical evaluations indicate that \textit{HAMMER} stimulates model ``curiosity'' and consistently achieves a 3\% to 4\% average accuracy gain across diverse inference benchmark.
\end{abstract}

\vspace{-4mm}
\section{Introduction}
\vspace{-2mm}
Recently, Reinforcement Learning with Verifiable Rewards (RLVR) has emerged as a powerful tool for enhancing complex reasoning in large language models (LLMs), significantly boosting their reasoning capabilities \citep{ReFT_44_44,OpenRFT_58_58,tulu_24_24}. During training, LLMs generate diverse responses to prompts and receive corresponding rewards \citep{seed1.5-vl_17_17,deepseekmath_grpo,kimi-1.5}. By learning from outcome reward, these models develop the ability to produce more comprehensive reasoning traces \citep{chen2025empiricalstudyelicitingimproving_9_9,deepseek-r1}, leading to improved performance on downstream tasks. The success of large reasoning models (e.g., OpenAI-o1 \citep{OpenAI-o1} and DeepSeek-R1 \citep{deepseek-r1}) demonstrates that RLVR effectively expands the capabilities of LLMs.

Group Relative Policy Optimization (GRPO) proposed by \citet{deepseekmath_grpo} is a key RLVR algorithm that extends Proximal Policy Optimization (PPO) proposed by \citet{ppo}, by sampling groups of responses to estimate group-relative advantages. Given reward $r$, group size $G$, policy ratio $\rho_t=\frac{\pi_\theta(o_t|q,o_{<t})}{\pi_{\theta_\text{old}}(o_t|q,o_{<t})}$ with bound $\varepsilon$, GRPO's objective function is
\vspace{-1mm}
$$
\mathcal{J}
(\theta) = 
    \mathbb{E}
    \left\{
    \frac{1}{G}
    \sum_{i=1}^G
        \frac{1}{|o_i|}\sum_{t=1}^{|o_i|}\min \left(
        \rho_{
        i,t} \hat{A}_{
        i,t}, clip\left(
            \rho_{
            i,t}, 1-\varepsilon, 1+\varepsilon
        \right)\hat{A}_{
        i,t}
    \right) - \beta \mathbb{D}_{
    \text{KL}}
    \right\},
$$
where KL divergence to reference policy is $\mathbb{D}_{\text{KL}}$ with penalty factor $\beta$. 
The normalized advantage is $\hat{A}_{i,t} = \frac{r_{i,t} - \text{mean}(r_{i,t})}{\text{std}(r_{i,t})}$. The expectation $\mathbb{E}$ follows $(q,a)\sim \mathcal{X}$ and $\{o_i\}_{i=1}^G \sim \pi_\theta(\cdot|q)$.
Subsequently, variants of GRPO, like Decoupled Clip and Dynamic sAmpling Policy Optimization (DAPO) \citep{dapo}, were proposed to optimize the GRPO.

% \shancun{[shancun comment] 
Beyond optimization algorithms, some works explore \textit{data-centric} strategies to improve efficiency. Inspired by human education, 
Curriculum Learning (CL) has been applied to LLM reinforcement \citep{Bengio2009CurriculumL,CL_Survey}, most studies rely on difficulty-based sequencing of Chain-of-Thought (CoT) annotations \citep{CL-E2H,qiu2025wisdom_E2H}. Such approaches typically mimic ``easy-to-hard'' progressions but require costly difficulty assessments, often via \textit{pass@k} testing or advanced-model labeling (e.g., OpenAI-o1 \citep{OpenAI-o1}, Deepseek-R1 \citep{deepseek-r1}) and suffer from local optimization. We consider adopting the diversity order, but diversity-based classical methods such as Coreset Selection (CS) \citep{_33_33_,_60_60_,_38_38_,task_cs_ft} are all sampling methods whose reduction-oriented design leads to performance bottlenecks \citep{cs_bottleneck}. 
% }

% Beyond optimization algorithms, some works explore \textit{data-centric} strategies to improve efficiency. Coreset selection (CS) accelerates training by selecting compact, representative subsets \citep{_33_33_,_58_58_,_60_60_,_65_65_,_15_15_,_38_38_,task_cs_ft}, but its reduction-oriented design leads to performance bottlenecks \citep{cs_bottleneck}. 
% % \red{Our Hamiltonian ordering instead adopts a curriculum learning (CL) perspective, leveraging semantic diversity to guide training and exposing models to varied samples early for more stable, exploratory learning.} 
% While Curriculum Learning (CL), inspired by human education, has been applied to LLM reinforcement \citep{Bengio2009CurriculumL,CL_Survey}, most studies rely on difficulty-based sequencing of Chain-of-Thought (CoT) annotations \citep{CL-E2H,qiu2025wisdom_E2H}. Such approaches typically mimic ``easy-to-hard'' progressions but require costly difficulty assessments, often via \textit{pass@k} testing or advanced-model labeling (e.g., OpenAI-o1 \citep{OpenAI-o1}, Deepseek-R1 \citep{deepseek-r1}).

\vspace{-2mm}
\subsection{Motivation}
\vspace{-2mm}

Reinforcement learning with LLMs often exhibits high variance and unstable convergence, particularly in the early stages of training \citep{passk_training}. Traditional curriculum learning \citep{CL_Survey} typically follows an ``easy-to-hard'' strategy \citep{CL-E2H,qiu2025wisdom_E2H}. 
% \blue{
However, in RLVR, such naive difficulty-based training often fails: the model quickly exploits easy samples for consistent rewards, while harder ones incur repeated penalties. This early imbalance discourages exploration, leading the policy to overfit to easy problems early and become trapped in local optima, ultimately slowing convergence.
Our ablation study confirms this inefficiency (Figure~\ref{fig: e2h}).  
% }
To improve training, we propose a different perspective: \emph{diversity can effectively guide RLVR}. Presenting semantically diverse samples early allows the model to explore the input space more thoroughly, reduce the generalization gap, and accelerate convergence, as theoretically justified in Section~\ref{sec: theoretical analysis}.  
In short, we transform diversity from a static dataset property to an active principle for curriculum design in LLM reinforcement learning.

\vspace{-2mm}
\subsection{Our Approach and Contributions}
\vspace{-2mm}
% \blue{
In this paper, we present a novel and effective schema, \textit{\underline{H}amiltonian Curiosity \underline{A}ug\underline{M}ented Large Language \underline{M}od\underline{E}l \underline{R}einforcement (HAMMER)}, 
which transfers diversity metrics from large model data evaluation into the dynamic process of reinforcement learning. 
The schema consists of two main components. 
First, it leverages the backbone LLM to obtain semantic similarity embeddings. 
Compared to external embedding models, this approach generates sentence representations that are more consistent with the model’s internal training dynamics. 
Second, the embeddings are used to compute pairwise semantic similarity, and a Hamiltonian Curiosity Order is constructed to define a curriculum learning sequence. 
This process can be viewed as solving a Hamiltonian cycle that minimizes semantic similarity, enabling the model to greedily encounter the most diverse samples early in training. 
As a result, the model can achieve partial convergence on some samples early, improving the stability of reinforcement learning.
% }

From a learning-theoretic perspective, we derive the generalization bound of \textit{HAMMER}. 
Theorem~\ref{thm: muti-stages training support condition} shows that early diverse training does not compromise the optimal policy, 
while Theorem~\ref{thm: local risk optimization} demonstrates that diverse subsets effectively tighten the generalization bound. 
% Moreover, Theorem~\ref{thm: min DCS iff min similarity} establishes that \textit{HAMMER}’s semantic diversity path optimization shares the same objective as maximizing the dataset diversity measure, which is the overall likelihood of large dissimilarities between a sample and the remaining dataset (formally defined in Equation~\ref{eq: definition of mu_DCS}).
Moreover, Theorem~\ref{thm: min DCS iff min similarity} establishes that optimizing \textit{HAMMER}’s semantic diversity path is equivalent to maximizing the dataset diversity score, which captures the overall likelihood that a sample substantially differs from the rest of the dataset (see the formal definition in Equation~\ref{eq: definition of mu_DCS}).
Extensive experiments validate our theoretical analysis and confirm its alignment with empirical results.

\textbf{Contributions}. In summary, this paper makes the following contributions:  
\vspace{-2mm}
\begin{itemize}[leftmargin=*]
    \item 
    % To enhance the stability and sample efficiency of LLM reinforcement learning, we integrate curriculum learning with sample diversity %(commonly used in data evaluation) 
    % and propose a novel schema, \textit{Hamiltonian Curiosity Augmented Large Language Model Reinforcement (HAMMER)}. \textit{HAMMER} leverages the minimum semantic Hamiltonian path, termed the \textit{Hamiltonian Curiosity Order}, to stimulate the model’s ``curiosity'' in the early stage, thereby accelerating convergence and optimization. 
    To improve the stability and sample efficiency of reinforcement learning, we introduce \textit{HAMMER}, a novel curriculum learning schema. \textit{HAMMER} structures the training sequence using a minimum semantic Hamiltonian path, termed the \textit{Hamiltonian Curiosity Order}. This path stimulates exploratory behavior (``curiosity'') in early training phases, leading to accelerated convergence and more stable optimization.
    
    \item %We design an efficient heuristic algorithm for computing the Hamiltonian Curiosity Order. This sample reordering approach achieves comparable benefits to expensive difficulty-based curriculum RL, but at a fraction of the cost.  
    We develop an efficient heuristic algorithm to compute the Hamiltonian Curiosity Order. This sample reordering approach delivers performance gains comparable to computationally expensive difficulty-based curriculum reinforcement learning, but with significantly lower overhead.
    
    \item From a theoretical perspective, we prove that \textit{HAMMER} preserves the optimal policy while promoting convergence by tightening the generalization bound with a small set of diverse samples. 
    Moreover, we show that the minimum semantic similarity cycle in HAMMER aligns with maximizing the dataset diversity score.
    % \red{diversity criterion $\mu_{\text{DCS}}$. }
    \item %Extensive experiments demonstrate that integrating \textit{HAMMER} into existing RLVR algorithms (e.g., DAPO and GRPO) consistently improves sample efficiency and typically yields a $3$--$4\%$ accuracy gain.  
    Extensive experiments demonstrate that integrating \textit{HAMMER} into RLVR algorithms like DAPO and GRPO consistently enhances sample efficiency and 
    %delivers a steady accuracy improvement of 3–4%.
    achieves accuracy gains of 3–4\%.
\end{itemize}

\section{Method Overview}
\vspace{-1mm}
In this paper, we propose \textit{Hamiltonian Curiosity Augmented Reinforcement (HAMMER)}, a novel training schema for LLMs comprising two key components.
\begin{figure}[t]
  \centering
  % \vspace{-6mm}
    \includegraphics[width=\linewidth]{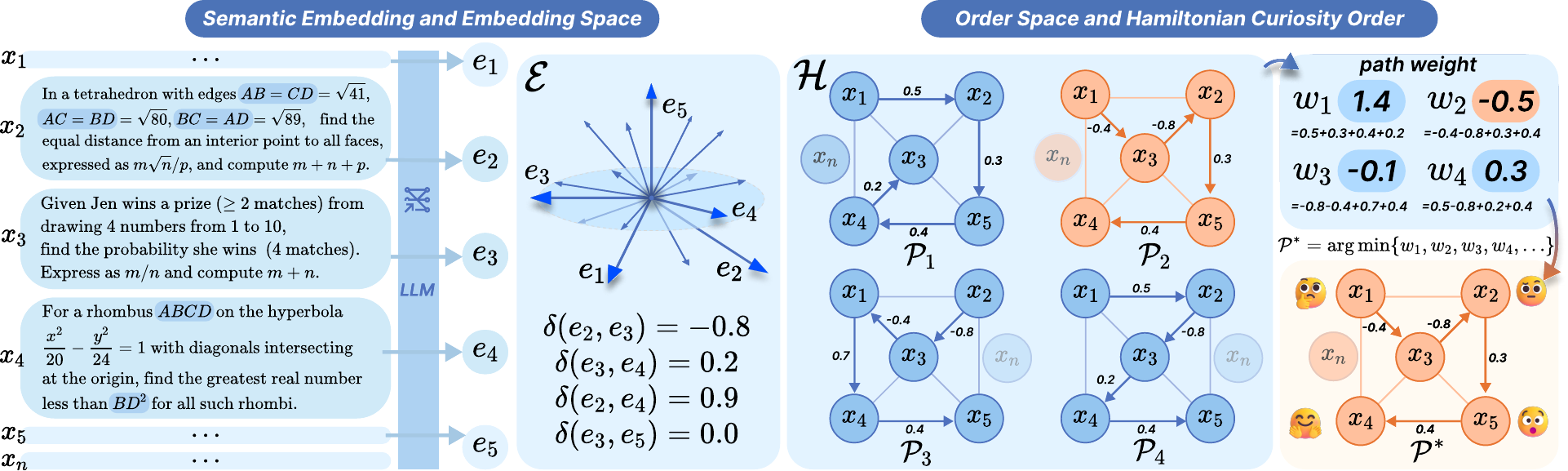}
    \vspace{-4mm}
    \caption{Overview of \textit{HAMMER}. Given dataset 
    $\mathcal{X}=\{x_i\}_{i=1}^n$, forward propagation through the backbone model yields sentence embeddings $\{e_i\}_{i=1}^n$, where similar ones are closer in embedding space $\mathcal{E}$ with larger similarity $\delta$ (e.g., $x_2, x_4$). Pairwise similarities form $\{\delta(e_i, e_j)\}_{n \times n}$, a complete graph. 
    All paths of the graph consists the Order Space $\mathcal{H}$.
    The path $\mathcal{P}^* \in \mathcal{H}$ with minimum similarity provides the \textit{Hamiltonian Curiosity Order}.}
    \label{fig: overview}
  % \vspace{-3mm}
\end{figure}

\vspace{-3mm}
\paragraph{Semantic Embedding} 
Sentence embeddings are obtained directly from the forward propagation of the backbone LLM, ensuring that the representation space reflects the model's own understanding of input text. 
Unlike embeddings derived from external models, this approach leverages the latent semantic structure captured by the backbone LLM itself, thereby reducing potential mismatch between training signals and the model's internal representation \citep{llm2vec}. 
Pairwise similarities between sentence embeddings define the embedding space, which can be represented as a similarity matrix $M=\{\delta(e_i, e_j)\}_{n \times n}$. 

\vspace{-3mm}
\paragraph{Hamiltonian Curiosity Order} 
Semantic similarity matrix $M$ can be viewed as a complete graph over $n$ samples, where every edge weight corresponds to semantic proximity. 
All possible sample orderings in this graph form the order space, containing $n!$ distinct paths. 
Order space provides a rich combinatorial structure for exploring different training sequences. 
Within the order space, we compute the Hamiltonian cycle of minimum cumulative similarity by Algorithm~\ref{alg: hamilton cycle}, which we call the \textit{Hamiltonian Curiosity Order}. 
This ordering intentionally prioritizes transitions across semantically dissimilar samples, thereby fostering a sense of ``curiosity'' in the early stages of reinforcement learning. 
Such curiosity-based ordering prevents premature overfitting to narrow semantic clusters, exposes the model to a broader spectrum of knowledge, and encourages more balanced exploration. 
As training proceeds, this induced diversity in trajectories helps smooth optimization dynamics and accelerates convergence, ultimately improving both stability and generalization of the backbone LLM under reinforcement learning.
We theoretically justify this intuition in Section~\ref{sec: theoretical analysis}. 
Specifically, Theorem~\ref{thm: muti-stages training support condition} shows that diverse subsets preserve the optimal policy during reinforcement learning. 
Theorem~\ref{thm: local risk optimization} further demonstrates that such subsets greedily minimize the generalization error bound. 
Finally, Theorem~\ref{thm: min DCS iff min similarity} establishes that the minimum semantic Hamiltonian cycle corresponds to maximizing the diversity measure $\mu_{\text{DCS}}$ (defined in Section~\ref{eq: definition of mu_DCS}).

% \section{Hamiltonian Curiosity Augmented Large Language Model Reinforcement}
\section{Methodology}
\subsection{Sentence Embedding and Similarity}
\label{sec: sentence embedding and similarity}
Common text similarity metrics include TF-IDF \citep{TF-IDF}, BLEU \citep{BLEU}, ROUGE-L \citep{ROUGE-L}, and semantic vector similarity.
While external embedding models often yield effective sentence embeddings for downstream tasks such as retrieval and classification, they may be misaligned with the backbone model’s internal representations \citep{llm2vec}. 

\begin{definition}[Sentence Embedding Space]
Given a dataset $\mathcal{X} = \{x_1, x_2, \dots, x_n\}$, 
a sentence embedding is a mapping
$
f: \mathcal{X} \to \mathbb{R}^d
$
(i.e., $f(x_i) = e_i$).
The embedding space is
$
\mathcal{E} = \{f(x) : x \in \mathcal{X}\} \subset \mathbb{R}^d,
$
with similarity typically measured by cosine similarity:
$
\delta(e_i,e_j) = \frac{\langle e_i, e_j \rangle}{\|e_i\| \|e_j\|}.
$
\end{definition}

In practice, to ensure consistency, we derive embeddings directly from the backbone LLM \citep{llm2vec}.  
Given a sentence $x \in \mathcal{X}$, a forward pass produces hidden states $\{h_t\}_{t=1}^{|x|}$, from which the embedding $e$ is obtained either by mean pooling over all tokens (i.e., 
$
e=\frac{1}{|x|} \cdot  \sum_{t=1}^{|x|} h_t
$
), yielding a compact vector.
% $x$. \shancun{In our experiment, we used ...}

\begin{example}
\label{ex: embedding space}
As illustrated in Figure~\ref{fig: overview}, each sentence in $\mathcal{X}$ is mapped into the embedding space $\mathcal{E}$ through the LLM forward pass, and similarity $\delta$ reflects semantic closeness. 
For instance, $x_2$ is closer in meaning to $x_4$ but more distinct from $x_3$, and their embeddings capture these relationships.
\end{example}

\subsection{Hamiltonian Curiosity Data Reorder}

% \begin{definition}[Order Space]
% Given a dataset $\mathcal{X} = \{x_1, x_2, \dots, x_n\}$ with sentence embeddings $\{e_1, e_2, \dots, e_n\}$, 
% let the semantic similarity matrix be
% $
% M_{n \times n} = \Big\{\tfrac{\langle e_i, e_j \rangle}{\|e_i\|\cdot\|e_j\|}\Big\}_{n \times n}.
% $
% Matrix $M$ can be viewed as the adjacency matrix of a complete weighted graph $\mathcal{G}=(\mathcal{X},\mathcal{X}^2, \delta)$, where edge weight $\delta(e_i,e_j)$ denotes semantic proximity.  
% The order space $\mathcal{H}(\mathcal{X})$ is the set of all possible paths in $\mathcal{G}$, i.e.,
% $$
% \mathcal{H}(\mathcal{X}) = \Big\{ \tau = (x_{\tau(1)},x_{\tau(2)},\dots,x_{\tau(n)}) : \tau \in \{1,2,\dots,n\} \Big\},
% \quad
% \text{ where $|\mathcal{H}(\mathcal{X})|=n!$}
% $$
% \end{definition}

\begin{definition}[Order Space]
\label{def: order space}
Given a dataset $\mathcal{X} = \{x_i\}_{i=1}^n$ with embeddings $\{e_i\}_{i=1}^n$, 
and let the semantic similarity matrix be
$
M_{ij} = \frac{\langle e_i, e_j\rangle}{\|e_i\|\|e_j\|}.
$
Interpret $M$ as the adjacency matrix of a complete weighted graph $\mathcal{G}=(\mathcal{X},E,\delta)$, 
where $E = \{(x_i,x_j): x_i,x_j\in \mathcal{X}\}$ and $\delta(x_i,x_j) = M_{ij}$.  
The order space $\mathcal{H}(\mathcal{X})$ is the set of all possible sequences of the samples in $\mathcal{X}$, i.e.,
\vspace{-1mm}
\[
\mathcal{H}(\mathcal{X}) = \Big\{\mathcal{P}=(x_{\tau_1},\dots,x_{\tau_n}): \tau \text{ is a permutation of } \{1,\dots,n\} \Big\},
% \quad |\mathcal{H}(\mathcal{X})| = n!,
\]
where each sequence $\mathcal{P}$ corresponds to a path in $\mathcal{G}$ that visits every node exactly once.
\end{definition}

% \shancun{
\begin{example}
In Figure~\ref{fig: overview}, with five samples set $\mathcal{X}_{eg}=\{x_1, x_2, x_3, x_4, x_5\}$, the order space $\mathcal{H}(\mathcal{X}_{eg})$ contains $5! = 120$ possible sequences. The figure illustrates four representative orders $\mathcal{P}_1$, $\mathcal{P}_2$, $\mathcal{P}_3$, $\mathcal{P}_4$ $\in$ $\mathcal{H}(\mathcal{X}_{eg})$ .  
In $\mathcal{P}_1$, RL training proceeds in the order $x_1 \rightarrow x_2 \rightarrow x_5 \rightarrow x_4 \rightarrow x_3$.
\end{example}
% }

% \begin{example}
% In Figure~\ref{fig: overview}, with five samples the order space $\mathcal{H}$ contains $5! = 120$ possible sequences. The figure illustrates four representative orders $\mathcal{P}_1$, $\mathcal{P}_2$, $\mathcal{P}_3$, $\mathcal{P}_4$ $\in$ $\mathcal{H}$ over $\{x_1, x_2, x_3, x_4, x_5\}$.  
% In $\mathcal{P}_1$, RL training proceeds in the order $x_1 \rightarrow x_2 \rightarrow x_5 \rightarrow x_4 \rightarrow x_3$.
% \end{example}

While the order space $\mathcal{H}(\mathcal{X})$ contains $n!$ paths, we select a single path, called the \emph{Hamiltonian Curiosity Order}.

\begin{definition}[Hamiltonian Curiosity Order]
Given a path $\mathcal{P} \in \mathcal{H}(\mathcal{X})$, its cumulative similarity is defined as
$
w(\mathcal{P}) = \sum_{k=1}^{n-1} \delta\big(e_{\mathcal{P}_k}, e_{\mathcal{P}_{k+1}}\big).
$
The \textit{Hamiltonian Curiosity Order} is the Hamiltonian path $\pi^*$ that minimizes this cumulative similarity
$
\mathcal{P}^* = \arg\min_{\mathcal{P} \in \mathcal{H}(\mathcal{X})} w(\mathcal{P}).
$ 
\end{definition}

Equivalently, $\mathcal{P}^*$ corresponds to a Hamiltonian cycle of minimum weight in $\mathcal{G}$ (Definition~\ref{def: order space}), which intentionally prioritizes traversals across semantically dissimilar samples. 

\begin{example}
\label{ex: hamiltonian curiosity order}
In Figure~\ref{fig: overview}, like Example~\ref{ex: embedding space}, we consider samples $\{x_1,x_2,x_3,x_4,x_5\}$, which are embedded into $\{e_1,e_2,e_3,e_4,e_5\}$.  
The similarity matrix
\[
% M(\mathcal{X})=
\begin{pmatrix}
1.0 & 0.5 & -0.4 & 0.7 & 0.8 \\
0.5 & 1.0 & -0.8 & 0.9 & 0.3 \\
-0.4 & -0.8 & 1.0 & 0.2 & -0.3 \\
0.7 & 0.9 & 0.2 & 1.0 & 0.4 \\
0.8 & 0.3 & -0.3 & 0.4 & 1.0
\end{pmatrix}
\]

represents the weighted complete graph.  
Among the $5!=120$ possible orders, the path $\mathcal{P}_2$, i.e., $x_2\!\to\!x_3\!\to\!x_5\!\to\!x_4\!\to\!x_1$, yields the minimum cumulative similarity $w_2=-0.4-0.8+0.3+0.4=-0.5$, defining $\mathcal{P}^*=\mathcal{P}_2$.  
This \textit{Hamiltonian Curiosity Order} ensures that the traversal moves across semantically diverse regions, thereby maximizing the diversity measure $\mu_{\text{DCS}}$ (Equation~\ref{eq: definition of mu_DCS}).
\end{example}

\begin{algorithm}[t]
    \caption{Hamiltonian Cycle with Minimum Semantic Similarity}
    \label{alg: hamilton cycle}
    \centering
    \small
    \begin{algorithmic}[1]
    \Require 
        dataset $\mathcal{X}=\{x_i\}_{i=1}^n$ with embeddings $\{e_i\}_{i=1}^n$, 
        similarity matrix $M_{n\times n}$,
        expand factor $\eta$.
    \Ensure 
        reordered dataset $\mathcal{X}'$ with minimum semantic similarity.

    \State $\mathcal{X}' \gets $ reorder $\mathcal{X}$ by \Call{HeuristicHamilton}{$M$, $\eta$} \label{line:call}
    
    \Function{HeuristicHamilton}{$W$, $\eta$} 
        \State $\mathcal{P}^* \gets \emptyset$, $w^* \gets -\infty$ \label{line:init}
        \For{$t = 1 \text{ to } \lfloor n/2 \rfloor$} \label{line:for_restart}
            \State $\mathcal{P} \gets$ a random $x_0 \in \mathcal{X}$, $\mathcal{V} \gets \{x_0\}$ \label{line:init_path}
            \While{$|\mathcal{P}| < n$} \label{line:while_greedy}
                \State $x' \gets$ last element of $\mathcal{P}$ \label{line:last_node}
                \State $x^* \gets$ randomly select one of the top-$\eta$ smallest in $\{(M_{x',z}, z) : z \in \mathcal{X} \land z \notin \mathcal{V}\}$ \label{line:select_next}
                \State $\mathcal{P} \gets \mathcal{P} \cup \{x^*\}, \mathcal{V} \gets \mathcal{V} \cup \{x^*\}$ \label{line:update_path}
            \EndWhile
            \State $w \gets \sum_{i=1}^{n-1} M_{\mathcal{P}_i, \mathcal{P}_{i+1}}$ \label{line:compute_weight}
            \If{$w > w^*$} \label{line:check_best}
                $w^* \gets w$, $\mathcal{P}^* \gets \mathcal{P}$ 
                \label{line:update_best}
            \EndIf
        \EndFor
        \State \Return $\mathcal{P}^*$ \label{line:return_best}
    \EndFunction
    \end{algorithmic}
\end{algorithm}

% 
% To obtain the \emph{Hamiltonian Curiosity Order} over the semantic similarity matrix via dynamic or integer programming is intractable for large datasets, being NP-hard \citep{NP-Hard}. Instead, we adopt an \emph{$\eta$-greedy heuristic} to efficiently approximate the minimum semantic similarity cycle, as detailed in Algorithm~\ref{alg: hamilton cycle}.
% Algorithm~\ref{alg: hamilton cycle} generates the \textit{Hamiltonian Curiosity Order} via a greedy heuristic. The algorithm maintains a global best path and its cumulative semantic similarity in line~\ref{line:init} and performs multiple random restarts to explore diverse candidate paths in lines~\ref{line:for_restart}--\ref{line:update_best}. Each restart begins with a randomly selected starting node and an initialized visited set (line~\ref{line:init_path}). At each step, the next node is chosen from the top-$\eta$ least similar unvisited nodes to encourage transitions across semantically distant samples, after which the current path and visited set are updated in lines~\ref{line:while_greedy}--\ref{line:update_path}. Upon completing a path, its total semantic similarity is computed (line~\ref{line:compute_weight}) and compared with the global best, updating it if superior line~\ref{line:check_best}. After all restarts, the algorithm returns the path with minimal semantic similarity as the Hamiltonian Curiosity Order in line~\ref{line:return_best}.

% \shancun{
To obtain the \emph{Hamiltonian Curiosity Order} over the semantic similarity matrix via dynamic or integer programming is intractable for large datasets, being NP-hard \citep{NP-Hard}. Instead, we propose an \emph{$\eta$-greedy heuristic search} ($\eta$-GHS) to efficiently approximate the minimum semantic similarity cycle, as detailed in Algorithm~\ref{alg: hamilton cycle}. 
The algorithm maintains a global best path $\mathcal{P}^*$ and its cumulative semantic similarity $w^*$ (line~\ref{line:init}). Concretely, $\eta$-GHS performs multiple random restarts to explore diverse candidate paths ( lines~\ref{line:for_restart}--\ref{line:update_best}), where each restart begins with a randomly selected starting node as the starting path $\mathcal{P}=\{x_0\}$ and an initialized visited set $\mathcal{V}=\{x_0\}$ (line~\ref{line:init_path}). After restarting, the next node $x^*$ is chosen from the top-$\eta$ least similar unvisited nodes to encourage transitions across semantically distant samples. Then the current path and visited set are updated (lines~\ref{line:while_greedy}--\ref{line:update_path}). Upon completing a path $\mathcal{P}$, its total semantic similarity $w$ is computed (line~\ref{line:compute_weight}) and compared with the global best $w^*$, updating it if superior (line~\ref{line:check_best}). After all restarts, the algorithm returns the path with minimal semantic similarity as the \textit{Hamiltonian Curiosity Order} (line~\ref{line:return_best}.)
% }
Algorithm~\ref{alg: hamilton cycle} obtains a minimal semantic similarity cycle via greedy search, which is equivalent to early-stage diversity, as formalized in Theorem~\ref{thm: min DCS iff min similarity} in Section~\ref{sec: theoretical analysis}. The computation complexity of Algorithm~\ref{alg: hamilton cycle} is $\mathcal{O}(n^2)$.

\subsection{Training on Hamiltonian Curiosity Ordered Dataset}
When trained on the Hamiltonian curiosity ordered dataset, the model greedily converges to a tighter generalization bound through subset-based training, achieving faster convergence toward the optimal policy than direct training on a shuffled dataset. 
We provide a theoretical justification for this phenomenon in Section~\ref{sec: theoretical analysis}. 
While such a greedy training scheme may bring little benefit to supervised learning with strong signals, it proves highly effective in the unstable setting of LLM reinforcement learning, where supervision is inherently weak. 
Example~\ref{ex: hammer advantages} further illustrates this idea: by greedily introducing diverse samples, \textit{HAMMER} accelerates convergence and yields smoother training dynamics. 
% The overall reinforcement training procedure is shown in Algorithm~\ref{alg: overall traininig procedure}.

\vspace{-2mm}
\section{Theoretical Analysis}
\vspace{-1mm}
\label{sec: theoretical analysis}
In this chapter, we show that training on diverse subsets reduces generalization error without losing the optimal policy, forming the basis of \textit{HAMMER}’s early diverse training, and that finding the minimal Hamiltonian cycle aligns with maximizing diversity.

\subsection{Preliminary}
\begin{definition}[Optimal Policy]
Let $\mathcal{X}$ denote the sample space (e.g., state-action pairs in RL), and $\Pi$ the set of all candidate policies. For $\pi \in \Pi$, define a bounded loss function $\mathcal{L}: \Pi \times \mathcal{X} \to \mathbb{R}$. The \textit{expected risk} of a policy $\pi$ is given by:
$
\mathcal{R}_{\mathcal{X}}(\pi) = \mathbb{E}_{x \sim \mathcal{X}}[\mathcal{L}(\pi, x)],
$
while the \textit{empirical risk} on a finite dataset $\mathcal{X}$ is:
$
\hat{\mathcal{R}}_{\mathcal{X}}(\pi) = \frac{1}{|\mathcal{X}|} \sum_{x \in \mathcal{X}} \mathcal{L}(\pi, x).
$
The optimal policy $\pi^*$ is defined as the minimizer of the expected risk
$
\pi^* = \arg \min_{\pi \in \Pi} \mathcal{R}_{\mathcal{X}}(\pi).
$
\end{definition}

\begin{definition}[Induced Policy Subset]
\label{def: induced policy subset}
Given the dataset $\mathcal{X}$ of size $n$, let $\mathcal{S} \subset \mathcal{X}$ and $\gamma$ be a tolerance factor. The policy subset induced by $\mathcal{S}$ is defined as
$
\Pi_{\mathcal{S}} = \left\{ \pi \in \Pi : \hat{\mathcal{R}}_{\mathcal{S}}(\pi) \leq \hat{\mathcal{R}}_{\mathcal{S}}^* + \gamma \right\} \subset \Pi,
$
where $\hat{\mathcal{R}}_{\mathcal{S}}^* = \min_{\pi \in \Pi} \hat{\mathcal{R}}_{\mathcal{S}}(\pi)$ denotes the minimal empirical risk over $\Pi$ on $\mathcal{S}$.
\end{definition}

\begin{definition}[Generalization Error]
    Given a policy $\pi$ and the optimal policy $\pi^*$, the generalization error of $\pi$ is defined as $\Delta_{\pi} = |\mathcal{R}(\pi)-\mathcal{R}(\pi^*)|$.
\end{definition}

\begin{definition}[Diversity Metric]
    Diversity metric $\mu$ is a measure from sample space to $\mathbb{R}$. The diversity $\mu(\mathcal{X})$ decreases as the sample similarity increases. 
\end{definition}

In this work, we adopt two recent diversity metrics: \textit{DCScore} \citep{diversity_in_synthetic_dataset} and \textit{n-gram} based distinct-$n$ method \citep{n-gram-diversity}.  
For a dataset $\mathcal{X} = \{x_1, \dots, x_n\}$ with embeddings $\{e_1, \dots, e_n\}$, let $M \in \mathbb{R}^{n\times n}$ be the semantic cosine similarity matrix with $M_{ij} = \langle e_i, e_j \rangle$. 
The \textit{DCScore} is defined as 
\vspace{-1mm}
\begin{equation}
\label{eq: definition of mu_DCS}
\mu_{\text{DCS}}(\mathcal{X}) = \mathbf{tr}
\left( \operatorname{sofmax} (M_{n\times n}) \right)
= \mathbf{tr} 
\left[
\left(
\frac{e^{M_{ij}}}{\sum_{j=1}^n e^{M_{ij}}}
\right)_{n\times n}
\right]
,
\end{equation}
where $\text{softmax}$ is applied row-wise and $\mathbf{tr}$ is the matrix trace.  
The $m$-gram metric measures lexical diversity by counting distinct $m$-grams across $\mathcal{X}$, let $G_m(x)$ be the multiset of $m$-grams in a sample $x$. The $m$-gram diversity is defined as 
\vspace{-1mm}
\begin{equation}
    \mu_{\text{NGM}}(\mathcal{X}) = \frac{|\{\, g : g \in G_m(x),  x \in \mathcal{X}\,\}|}{\sum_{x\in\mathcal{X}} |G_m(x)|}.
\end{equation}
Both $\mu\_{\text{DCS}}$ and $\mu\_{\text{NGM}}$ may decay with increasing sample size \citep{diversity_in_synthetic_dataset}, so an adjustment is $\mu(\mathcal{X})=|\mathcal{X}|^p \cdot \mu(\mathcal{X})$, where $p$ is a constant; following \citep{diversity_in_synthetic_dataset}, we set $p=0.5$.

\subsection{Key Theorems}
All proofs of the following theorems are detailed in Appendix~\ref{app: proof of theorems}. 
By the VC inequality \citep{proof_of_vc_inequality} (formally defined in Lemma~\ref{lem: vc inequality}), with probability at least $1-\delta$, for a policy class $\Pi$ with VC dimension $d$ and $n$ i.i.d. samples $\mathcal{S}$, the following inequality holds
\begin{equation}
    \label{eq: generation bound}
    \sup_{\pi \in \Pi} \left| \hat{\mathcal{R}}_{\mathcal{S}}(\pi) - \mathcal{R}(\pi) \right| \leq C\sqrt{\frac{d\log(n/d) + \log(1/\delta)}{n}}, \quad
\text{where $C > 0$ is some constant.}
\end{equation}
For short, denote $\rho = C\sqrt{\tfrac{d\log(n/d)+\log(1/\delta)}{n}}$, where $d$ is the VC-dimension and $n$ the sample size.

\begin{theorem}
\label{thm: muti-stages training support condition}
Given a subset $\mathcal{S} \subset \mathcal{X}$ of $n$ samples, let $\pi^*$ be the optimal policy on $\mathcal{X}$. There exists some $\gamma$ (i.e., $\gamma=2\rho$) such that
$
\pi^* \in \Pi_{\mathcal{S}}.
$
\end{theorem}

By Theorem~\ref{thm: muti-stages training support condition}, selecting a subset $\mathcal{S}$ from $\mathcal{X}$ that satisfies the $\gamma$-condition ensures that the optimal policy $\pi^*$ is preserved, thereby guaranteeing the optimality of the subset selection approach.

\begin{theorem}
\label{thm: local risk optimization}
    % Given a subset $\mathcal{S}$ of $n$ samples,
    % if $\gamma=2\rho$, then 
    % $\forall \pi \in \Pi_{\mathcal{S}}, \Delta_{\pi} \leq \mathcal{O}
    % \left(
    % \sqrt{\frac{d\log(n/d) + \log(1/\delta)}{n}}
    % \right).
    % $
    For a subset $\mathcal{S}$ of $n$ samples, when $\gamma=2\rho$,
\vspace{-1mm}
$$
\forall \pi \in \Pi_{\mathcal{S}}, \Delta_{\pi} \leq \mathcal{O}\left(\sqrt{\frac{d\log(n/d) + \log(1/\delta)}{n}}\right).
$$
\end{theorem}

% \begin{proof}
%     $\forall \pi \in \Pi_{\mathcal{S}}$, by Inequality~\ref{eq: uniform bound} and Definition~\ref{def: induced policy subset}, we have 
%     \vspace{-1mm}
%     $$
%     \mathcal{R}(\pi) \leq \hat{\mathcal{R}}_{\mathcal{S}}(\pi) + \rho \leq 
%     \left(
%     \hat{\mathcal{R}}_{\mathcal{S}}^* + \gamma
%     \right) + \rho = 
%     (\hat{\mathcal{R}}(\pi^*) + \gamma) + \rho
%     = \hat{\mathcal{R}}(\pi^*) + 3\rho
%     $$
%     \vspace{-1mm} 
%     Thus, we deduce $\Delta_{\pi} = |\mathcal{R}(\pi)-\mathcal{R}(\pi^*)| \leq 3\rho = \mathcal{O}
%     \left(
%     \sqrt{\frac{d\log(n/d) + \log(1/\delta)}{n}}
%     \right)
%     $.
% \end{proof}

The generalization error bound $\rho \propto \mathcal{O}(\sqrt{\log n / n})$, which decreases slowly; hence, the benefit of additional samples diminishes as $n$ grows, especially in unstable LLM reinforcement learning.
To this end, Theorems~\ref{thm: muti-stages training support condition} and \ref{thm: local risk optimization} demonstrate that, without sacrificing the optimal policy, optimizing over a subset can also effectively reduce the generalization error.

Thus, a more diverse subset $\mathcal{S}$ enables the empirical risk $\hat{\mathcal{R}}_{\mathcal{S}}$ to better approximate the true risk $\mathcal{R}$, and still enhancing generalization. We therefore partition $\mathcal{X}$ into $\mathcal{S}$ and $\mathcal{X}\setminus\mathcal{S}$, selecting $\mathcal{S}$ to maximize diversity, and adopt a two-stage training scheme. 
In LLM reinforcement learning, such early reduction of the generalization gap promotes convergence under high variance, since it can quickly decrease the generation error bound. This idea naturally generalizes to multi-stage: 
\label{muti-stages schema}
dividing $\mathcal{X}$ into $k$ subsets $\mathcal{S}_1 \subset \mathcal{X}, \mathcal{S}_2 \subset \mathcal{X}/\mathcal{S}_1,\dots, \mathcal{S}_k\subset \mathcal{X}/ \cup_{i=1}^{k-1}\mathcal{S}_i$, each maximizing diversity $\mu(\mathcal{S}_i)$, and training sequentially in $k$ stages. As $k$ grows large, this process converges to our proposed \textit{HAMMER}.

\begin{example}
\label{ex: hammer advantages}
In Figure~\ref{fig: example of diversified sample training}, 
\textit{HAMMER} leverages diverse subsets (e.g., $\mathcal{S}\subset\mathcal{X}$) to rapidly reduce generalization error, while ensuring that the candidate policy set $\Pi_{\mathcal{S}}$ still retains the optimal policy $\pi^*$. 
Unlike training on the full dataset $\mathcal{X}$, where the generalization risk closely follows the true risk surface (green curve), training on a subset $\mathcal{S}$ yields a sparser trajectory: the subset risk does not exactly match the original risk, nor does it directly converge to the global minimum. 
By greedily introducing diverse samples, \textit{HAMMER} accelerates convergence in LLM reinforcement learning and leads to smoother training dynamics.

% \begin{figure}[t]
%   \centering
%   % \vspace{-6mm}
%     \includegraphics[width=0.6\linewidth]{theoretical_example.png}
% \caption{Illustration of the benefits of introducing diversity early in training.}
%   \label{fig: example of diversified sample training}
%   % \vspace{-3mm}
% \end{figure}
\end{example}

% Example~\ref{ex: hammer advantages} illustrates how \textit{HAMMER}, by leveraging Theorem~\ref{thm: local risk optimization} with early-stage diversity training, rapidly tightens the generalization bound and stabilizes reinforcement learning. Diverse samples bring empirical risk closer to true risk, while even a small set of them effectively reduces generalization error, leading to more efficient optimization.
Example~\ref{ex: hammer advantages} shows how \textit{HAMMER}, via Theorem~\ref{thm: local risk optimization} and early diversity training, quickly tightens the generalization bound and stabilizes RL. Diverse samples align empirical with true risk and reduce generalization error, enabling more efficient optimization.

\begin{theorem}
\label{thm: min DCS iff min similarity}
    Given a dataset $\mathcal{X} = \{x_i\}_{i=1}^n$ with embeddings $\{e_i\}_{i=1}^n$, let $\mathcal{S} \subset \mathcal{X}$ and $|\mathcal{S}|=m$, $M(\mathcal{S}) \in \mathbb{R}^{m\times m}$ be the semantic cosine similarity matrix of $\mathcal{S}$ with $M_{ij}(\mathcal{S}) = \langle e_i, e_j \rangle$, then we have
    $$
    \max_{\mathcal{S}\subset \mathcal{X}, |\mathcal{S}|=m} \mu_{\text{DCS}}(\mathcal{S}) \iff \min_{\mathcal{S}\subset \mathcal{X}, |\mathcal{S}|=m} \sum_{i=1}^m\sum_{j=1}^m \mathbb{I}(i\neq j)\cdot M_{ij}(\mathcal{S}).
    $$
\end{theorem}

\begin{example}
\label{ex: maximize mu_DCS}
Let $\mathcal X=\{x_1,\dots,x_5\}$ and $M(\mathcal X)$ be the semantic-similarity matrix given in Example~\ref{ex: hamiltonian curiosity order}.
We compare Hamiltonian Curiosity Order $\mathcal{P}^*$ and another order $\mathcal{P}_1$  of the samples:
\[
\mathcal P^* = \mathcal{P}_2=(x_2,x_3,x_5,x_4,x_1) \quad \text{ and }
\quad
\mathcal P_1=(x_1,x_2,x_5,x_4,x_3).
\]
As defined in Equation~\ref{eq: definition of mu_DCS}, for each prefix of size $n$ we compute
\(\mu_{\text{DCS}}=\mathbf{tr}\big(\operatorname{softmax}(M)\big)\cdot n^p\)
with $p=1$.
As shown in Table~\ref{tab:mu_dcs_comparison}, we observe that for every subset ($n<5$), the Hamiltonian curiosity order $\mathcal P^*$ attains a strictly larger
\(\mu_{\text{DCS}}\) than $\mathcal P_1$, while both orders coincide on the full set ($n=5$).
Hence $\mathcal P_1$ maximizes sample diversity in early training stages.

% \begin{table}[h]
% % \small
% \centering
% \caption{Comparison of $\mu_{\text{DCS}}$ values for different prefix subsets under two orders.}
% \label{tab:mu_dcs_comparison}
% \begin{tabular}{c ccccc}
% \toprule
% \textbf{Prefix Subset} & $\{x_1\}$ & $\{x_1,x_2\}$ & $\{x_{i}\}_{i=1}^3$ & $\{x_{i}\}_{i=1}^4$ & $\{x_{i}\}_{i=1}^5$ \\
% \midrule
% $\mu_{\text{DCS}}$ of $\mathcal{P}^*$ & 1.00 & 3.43 & 5.59 & 6.78 & 8.35 \\
% $\mu_{\text{DCS}}$ of $\mathcal{P}_1$ & 1.00 & 2.49 & 3.96 & 5.24 & 8.35 \\
% \bottomrule
% \end{tabular}
% \end{table}
\end{example}

\vspace{-2mm}
\section{Evaluation}
\vspace{-1.5mm}
\subsection{Experimental Setup}
\vspace{-2mm}
Algorithms, data and experimental details are included in supplementary materials.
% the anonymous repository 
% \url{https://anonymous.4open.science/r/HAMMER-B17F}.
% \paragraph{Datasets}
% We evaluate HAMMER on four mathematical benchmarks: AIME 2024, AIME 2025, AMC 2023, and Olympiad. All models are trained on DeepScaleR, with datasets ordered by minimal similarity via Algorithm~\ref{alg: hamilton cycle}. For the main experiments, we compare against DAPO and GRPO trained on randomly shuffled data. On AIME 2024, AIME 2025, and AMC 2023, we generate $1$, $10$, and $100$ responses, repeat sampling 10 times, and report the average $pass@k$ and $cons@100$ $(k \in {1,10,100})$, where $pass@k$ measures solution accuracy (the probability that at least one out of $k$ attempts passes verification) and $cons@k$ measures response consistency (the probability that the majority-vote answer among $k$ attempts passes verification) \citep{deepseek-r1}. For Olympiad, due to its larger scale, we report only $pass@1$, $pass@10$, $pass@32$, and $cons@32$, which suffice for reliable estimation.
% Please see Appendix~\ref{app: experimental setup} for more details.
% \blue{
We evaluate \textit{HAMMER} on four mathematical benchmarks: AIME 2024, AIME 2025, AMC 2023, and Olympiad. Models are trained on DeepScaleR ordered by Algorithm~\ref{alg: hamilton cycle}. We compare against DAPO and GRPO trained on randomly shuffled data. For AIME 2024/2025 and AMC 2023, we report average \textit{pass@1, pass@10, pass@100} and $cons@100$, where $pass@k$ measures solution accuracy 
and $cons@k$ 
(frequency that at least one out of $k$ attempts passes verification)
(frequency that the majority-answer among $k$ attempts passes verification)
measures majority-vote consistency \citep{deepseek-r1}. For larger Olympiad, we report $pass@1$, $pass@10$, $pass@32$, and $cons@32$. See Appendix~\ref{app: experimental setup} for details.
% }

\vspace{-1mm}
\subsection{Main Experiment}
\vspace{-2mm}

\begin{table*}[t]
\centering
% \vspace{-6mm}
\caption{Main results comparing baseline (B) and \textit{HAMMER} (H), with $k=100$ for AIME 2024/2025 and AMC 2023, $k=32$ for Olympiad, and $\mathbf{Diff}=\mathbf{Avg}_{\text H}-\mathbf{Avg}_{\text B}$.}
\label{tab: main results}
\renewcommand{\arraystretch}{1.1}
\setlength{\tabcolsep}{3pt}
\scalebox{0.75}{ % 闪存
\setlength{\tabcolsep}{2.8mm}

\begin{tabular}{l l c c c c c c c c c c c}
\toprule
\multirow{2}{*}{\textbf{Method}}
& 
\multirow{2}{*}{\textbf{Dataset}}
& \multicolumn{2}{c}{pass@1} & \multicolumn{2}{c}{pass@10} & \multicolumn{2}{c}{pass@k} & \multicolumn{2}{c}{cons@k} & \multicolumn{2}{c}{Avg.} & 
\multirow{2}{*}{\textbf{Diff.}}
\\
% \cline{3-12}
\cmidrule(r){3-4} \cmidrule(r){5-6} \cmidrule(r){7-8} \cmidrule(r){9-10} \cmidrule(r){11-12}
 & & B & H & B & H & B & H & B & H & B & H & \\ 
\midrule
\multirow{4}{*}{Qwen3-1.7B-DAPO} & AIME 2024 & 36.3 & 39.3 & 61.0 & 64.7 & 69.0 & 74.3 & 43.3 & 43.3 & 52.4 & \cellcolor{blue!12}55.4 & \textcolor{ForestGreen}{+3.0} \\
 & AIME 2025 & 25.3 & 31.0 & 39.7 & 48.7 & 56.3 & 59.7 & 30.0 & 30.0 & 37.8 & \cellcolor{blue!12}42.3 & \textcolor{ForestGreen}{+4.5} \\
 & AMC 2023  & 64.2 & 68.9 & 83.3 & 85.1 & 90.3 & 91.3 & 74.7 & 77.1 & 78.1 & \cellcolor{blue!5}80.6 & \textcolor{ForestGreen}{+2.5} \\
 & Olympiad  & 51.7 & 53.5 & 64.0 & 65.3 & 67.3 & 68.3 & 56.6 & 56.6 & 59.9 & \cellcolor{blue!5}60.9 & \textcolor{ForestGreen}{+1.0} \\ \hline

\multirow{4}{*}{Qwen3-1.7B-GRPO} & AIME 2024 & 36.3 & 40.0 & 59.3 & 63.6 & 70.0 & 73.3 & 43.3 & 43.3 & 52.4 & \cellcolor{blue!5}55.1 & \textcolor{ForestGreen}{+2.7} \\
 & AIME 2025 & 24.7 & 26.3 & 40.3 & 44.3 & 50.7 & 59.7 & 30.0 & 30.0 & 36.4 & \cellcolor{blue!12}40.1 & \textcolor{ForestGreen}{+3.7} \\
 & AMC 2023  & 63.1 & 68.5 & 83.0 & 84.7 & 88.5 & 91.2 & 74.7 & 77.1 & 77.4 & \cellcolor{blue!12}80.4 & \textcolor{ForestGreen}{+3.0} \\
 & Olympiad  & 53.3 & 54.0 & 65.6 & 65.4 & 68.8 & 68.4 & 56.7 & 56.6 & 61.1 & 61.1 & \textcolor{black}{0.0} \\ \hline

\multirow{4}{*}{Qwen3-4B-DAPO} & AIME 2024 & 52.3 & 54.7 & 72.0 & 75.7 & 79.7 & 83.3 & 60.0 & 63.3 & 66.0 & \cellcolor{blue!12}69.3 & \textcolor{ForestGreen}{+3.3} \\
 & AIME 2025 & 39.7 & 43.7 & 51.7 & 60.7 & 63.0 & 63.3 & 46.7 & 53.3 & 50.3 & \cellcolor{blue!12}55.3 & \textcolor{ForestGreen}{+5.0} \\
 & AMC 2023  & 75.5 & 78.6 & 87.9 & 88.3 & 91.6 & 91.6 & 83.1 & 81.3 & 84.5 & 85.4 & \textcolor{ForestGreen}{+0.9} \\
 & Olympiad  & 62.4 & 63.1 & 72.8 & 74.2 & 75.5 & 76.6 & 62.5 & 64.3 & 68.3 & \cellcolor{blue!5}69.6 & \textcolor{ForestGreen}{+1.3} \\ \hline

\multirow{4}{*}{Qwen3-4B-GRPO} & AIME 2024 & 48.9 & 49.7 & 67.6 & 71.3 & 73.1 & 83.0 & 60.0 & 56.7 & 62.4 & \cellcolor{blue!5}65.2 & \textcolor{ForestGreen}{+2.8} \\
 & AIME 2025 & 40.0 & 43.7 & 54.7 & 60.3 & 60.0 & 66.3 & 53.3 & 50.0 & 52.0 & \cellcolor{blue!12}55.8 & \textcolor{ForestGreen}{+3.8} \\
 & AMC 2023  & 76.0 & 77.7 & 88.5 & 91.2 & 92.0 & 94.7 & 86.7 & 86.8 & 85.8 & \cellcolor{blue!5}87.6 & \textcolor{ForestGreen}{+1.8} \\
 & Olympiad  & 62.5 & 63.7 & 72.5 & 74.0 & 75.4 & 76.5 & 62.7 & 64.3 & 68.3 & \cellcolor{blue!5}69.6 & \textcolor{ForestGreen}{+1.3} \\
\bottomrule
\end{tabular}
}
\end{table*}

Main experiment is trained on the DeepScaleR using Qwen3-1.7B and Qwen3-4B as backbone models with DAPO and GRPO. The baseline adopts randomly shuffled training data, while \textit{HAMMER} leverages the \textit{Hamiltonian Curiosity Order}. 
After convergence, we evaluate the models with \textit{pass@1}, \textit{pass@10}, \textit{pass@100} and \textit{cons@100}.
As shown in Table~\ref{tab: main results},
\textit{HAMMER} achieves an average accuracy improvement of 3--4 \% over the baseline. The models not only improve pass rates but also enhance answer consistency. 
As model size increases, the performance gains of \textit{HAMMER} remain stable, demonstrating that \textit{HAMMER} effectively leverages semantic similarity to optimize training efficiency without diminishing with larger models.

\subsection{Training Dynamic}
\vspace{-1mm}
Figure~\ref{fig: qwen3-1.7b dapo training dynamic} and 
Figure~\ref{fig: qwen3-1.7b grpo training dynamic} present the \textit{pass@k} evaluation of Qwen3-1.7B trained on DeepScaleR across AIME 2024, AIME 2025, AMC  2023, and Olympiad ($k=8$ for AIME 2024, AIME 2025, AMC 2023, and $k=1$ for Olympiad). 
% \textit{HAMMER} consistently achieves better performance at the same step. 
% The advantage of GRPO with \textit{HAMMER} is less pronounced on Olympiad, it still shows a superior trend in the later stages of training,
% which can be attributed to the more diverse training induced in the early stages (see the analysis in Section~\ref{sec: theoretical analysis} and Section~\ref{sec: theory validation}).
\textit{HAMMER} consistently outperforms baselines at the same step. 
For GRPO on Olympiad, the gains are smaller but become evident in later stages.

\begin{figure}[htbp]
% \small
  \centering
  \vspace{-2mm}
  \includegraphics[width=\linewidth]{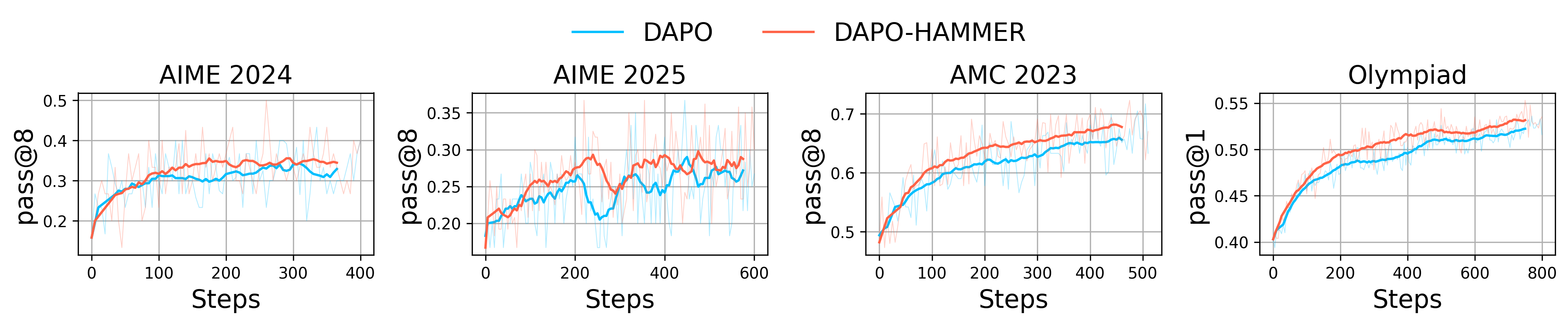}
  \vspace{-5mm}
  \caption{Validation of \textit{pass@k} over steps on Qwen3-1.7B DAPO (8192 context).}
  \label{fig: qwen3-1.7b dapo training dynamic}
  % \vspace{-2mm}
\end{figure}

\begin{figure}[htbp]
% \small
  \centering
  \vspace{-2mm}
  \includegraphics[width=\linewidth]{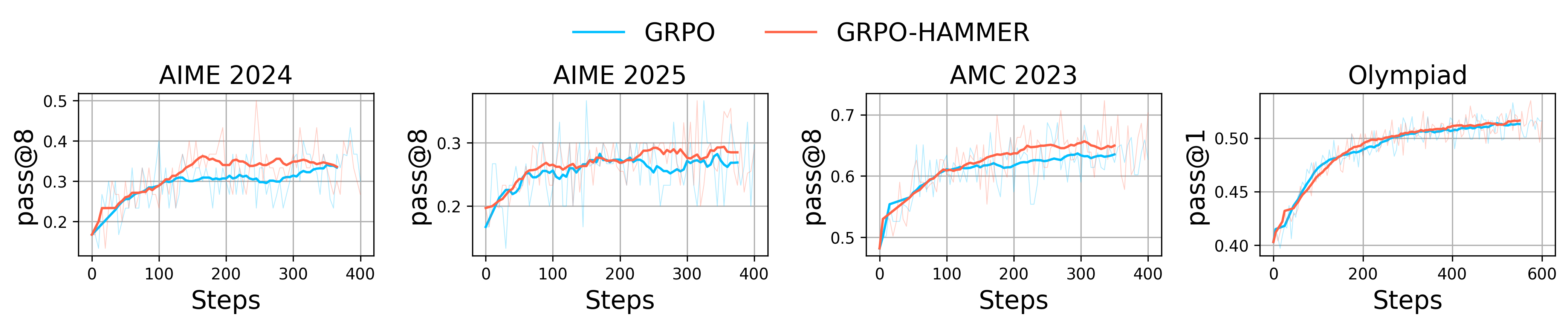}
  \vspace{-6mm}
 \caption{Validation of \textit{pass@k} over steps on Qwen3-1.7B GRPO (8192 context).}
  \label{fig: qwen3-1.7b grpo training dynamic}
  % \vspace{-2mm}
\end{figure}

% \vspace{-3mm}
\subsection{Ablation Study and Theoretical Validation}
\vspace{-1mm}
\paragraph{Zero-shot Performance} 
Table~\ref{tab: zero-shot performance} reports the zero-shot reasoning performance of the backbone models Qwen3-1.7B and Qwen3-4B on AIME 2024, AIME 2025, AMC 2023, and Olympiad. Combined with Table~\ref{tab: main results}, while RLVR yields about a 10\% improvement over the backbone models, \textit{HAMMER} achieves a 3--4\% gain solely through sample reordering.

\begin{table*}[b]
\small
\centering
\vspace{-5mm}
\caption{Zero-shot performance 
% on different backbone models (\(k=100\) 
for AIME 2024, AIME 2025, and AMC 2023; \(k=32\) for Olympiad).}
\label{tab: zero-shot performance}
\renewcommand{\arraystretch}{1.1}
\setlength{\tabcolsep}{3pt}
\scalebox{0.85}{
\setlength{\tabcolsep}{2.4mm}
\begin{tabular}{l c c c c c c c c c c c}
\toprule
\multirow{2}{*}{\textbf{Dataset}}
& \multicolumn{5}{c}{\textbf{Qwen3-1.7B}} & \multicolumn{5}{c}{\textbf{Qwen3-4B}} \\ 
\cmidrule(r){2-6} 
\cmidrule(r){7-11}
 & pass@1 & pass@10 & pass@k & cons@k & Avg. & pass@1 & pass@10 & pass@k & cons@k & Avg.  \\ 
\midrule
AIME 2024 & 15.7 & 39.3 & 58.7 & 20.0 & 33.4 & 26.3 & 42.7 & 55.3 & 40.0 & 41.1 \\
AIME 2025 & 18.7 & 26.7 & 28.7 & 35.3 & 23.3 & 17.0 & 28.7 & 35.3 & 23.3 & 26.1 \\
AMC 2023  & 47.7 & 62.5 & 72.4 & 50.6 & 58.3 & 52.8 & 67.6 & 76.3 & 60.2 & 64.2 \\
Olympiad & 42.1 & 53.3 & 55.7 & 43.6 & 48.6 & 45.5 & 55.3 & 58.3 & 46.5 & 51.4 \\
\bottomrule
\end{tabular}
}
\end{table*}

% \begin{table*}[t]
% \small
% \centering
% \caption{Zero-shot performance on different backbone models (\(k=100\) for AIME 2024, AIME 2025, and AMC 2023; \(k=32\) for Olympiad).}
% \label{tab: zero-shot performance}
% \renewcommand{\arraystretch}{1.1}
% \setlength{\tabcolsep}{3pt}
% \begin{tabular}{l c c c c c c c c c c c}
% \toprule
% \multirow{2}{*}{\textbf{Dataset}}
% & \multicolumn{5}{c}{\textbf{Qwen3-1.7B}} & \multicolumn{5}{c}{\textbf{Qwen3-4B}} \\ 
% \cmidrule(r){2-6} 
% \cmidrule(r){7-11}
%  & pass@1 & pass@10 & pass@k & cons@k & Avg. & pass@1 & pass@10 & pass@k & cons@k & Avg. & \\ 
% \midrule
% AIME 2024 & 15.7 & 39.3 & 58.7 & 20.0 & 33.4 & 26.3 & 42.7 & 55.3 & 40.0 & 41.1 \\
% AIME 2025 & 18.7 & 26.7 & 28.7 & 35.3 & 23.3 & 17.0 & 28.7 & 35.3 & 23.3 & 26.1 \\
% AMC 2023  & 47.7 & 62.5 & 72.4 & 50.6 & 58.3 & 52.8 & 67.6 & 76.3 & 60.2 & 64.2 \\
% Olympiad & 42.1 & 53.3 & 55.7 & 43.6 & 48.6 & 45.5 & 55.3 & 58.3 & 46.5 & 51.4 \\
% \bottomrule
% \end{tabular}
% \end{table*}

\vspace{-2mm}
\paragraph{Maximal Semantic Sample Order}
% While minimal semantic similarity ordering benefits reinforcement learning, we also examine maximal similarity ordering. 
% Following the main experiment, we study the validation of \textit{pass@8} on AIME 2024 with the maximal semantic Hamiltonian sample order on DeepScaleR with Qwen3-1.7B (8192 context), setting $M=-M$ so Algorithm~\ref{alg: hamilton cycle} still applies, akin to neighbor-based training \citep{neighbor_rl}. 
% As shown in Figure~\ref{fig: max hamilton}, \textit{HAMMER} remains superior. 
% Maximal similarity ordering outperforms random shuffle, but \textit{HAMMER}'s early training on diverse samples more robustly contracts the generalization bound, 
% reducing susceptibility to local optima even with later similar samples.
% \blue{
% While minimal semantic similarity ordering benefits RL, we also evaluate maximal similarity ordering. Using Qwen3-1.7B (8192 context) on AIME 2024, we validate $pass@8$ with maximal semantic Hamiltonian ordering on DeepScaleR, setting $M=-M$ so Algorithm~\ref{alg: hamilton cycle} applies, similar to neighbor-based training \citep{neighbor_rl}. 
% As shown in Figure~\ref{fig: max hamilton}, \textit{HAMMER} remains superior.
While minimal similarity ordering benefits RL, we also test maximal similarity ordering. On AIME 2024 with Qwen3-1.7B, we validate $pass@8$ using maximal semantic Hamiltonian ordering ((1)$M=-M$; (2) Algorithm~\ref{alg: hamilton cycle}), akin to neighbor-based training \citep{neighbor_rl}. As shown in Figure~\ref{fig: max hamilton}, \textit{HAMMER} remains superior.

\begin{figure}[t]
% \vspace{-10mm}
\centering
  \subfigure[Maximum Hamiltonian similarity ablation study.]{
    \includegraphics[width=0.45\linewidth]{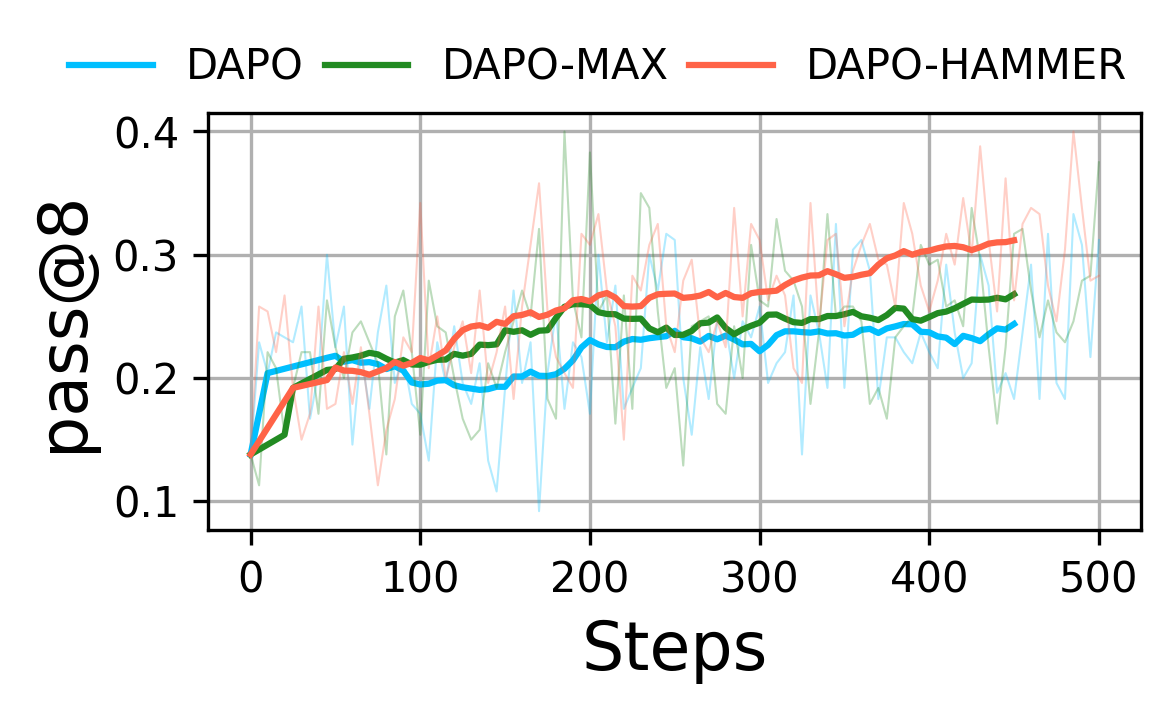} 
    \label{fig: max hamilton}
    \vspace{-2mm}
  }
  \hfill
  \subfigure[Difficulty-based ablation study.]{
    \includegraphics[width=0.45\linewidth]{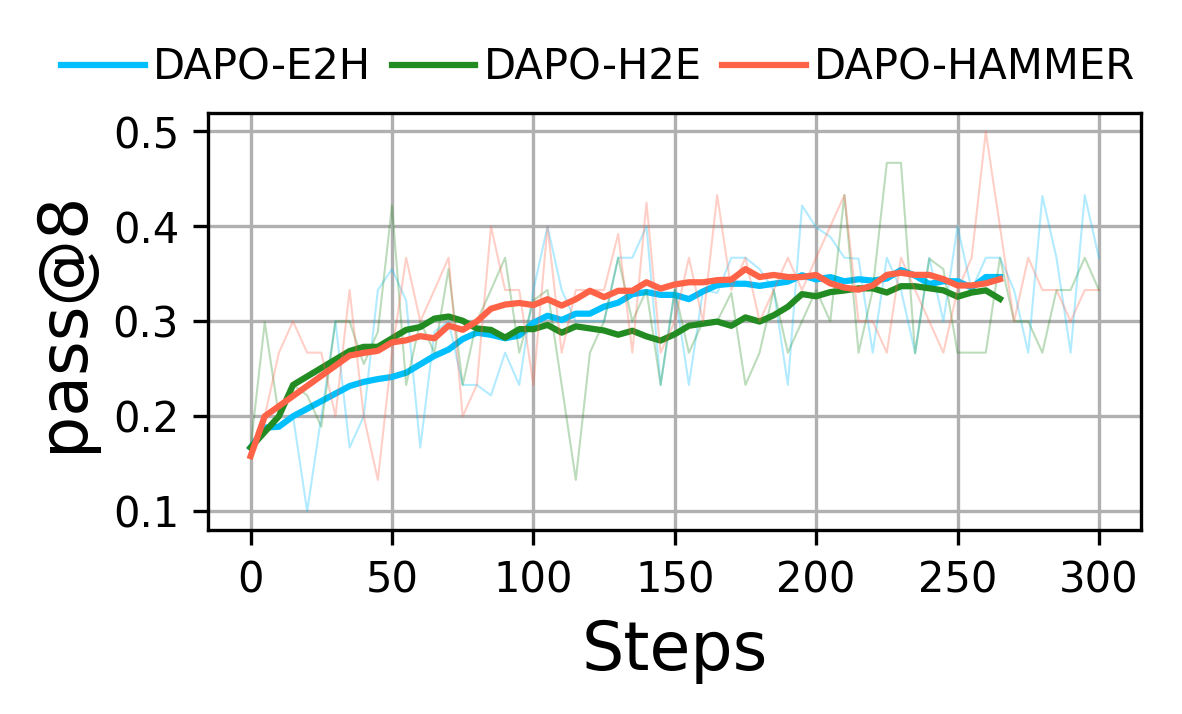}
    \label{fig: e2h}
    \vspace{-2mm}
  }
  \subfigure[Validation of \textit{pass@k} over steps on Qwen3-1.7B DAPO varing \textit{batch size} $\in \{16,32,64\}$.]{
  \includegraphics[width=\linewidth]{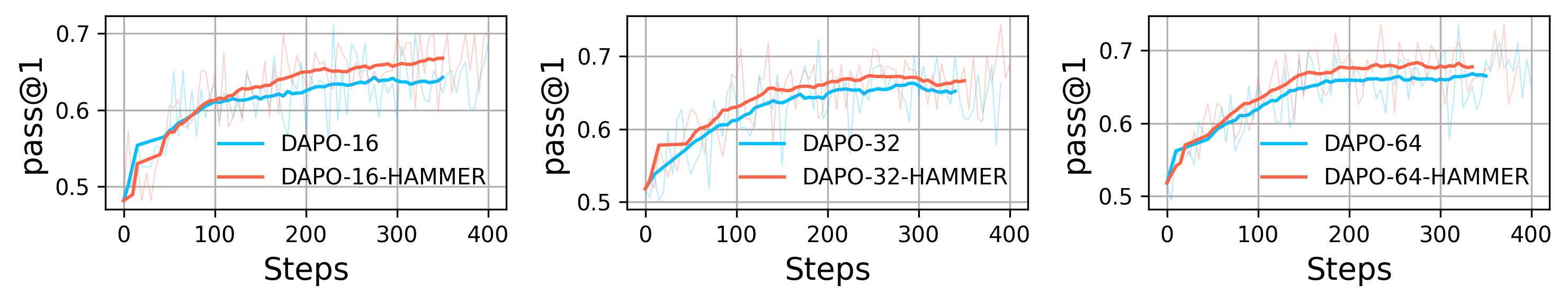}
  \label{fig: qwen3-1.7b dapo batchsize ablation}
  \vspace{-2mm}
  }
\vspace{-3mm}
\caption{Data order and batch size ablation study, where DAPO-MAX denote \textit{max semantic similarity} data order, 
DAPO-E2H and DAPO-H2E denote ``easy-to-hard'' and ``hard-to-easy'' data order.
}
\label{fig: max hamilton and e2h and batchsize}
\end{figure}

\begin{figure}[t]
% \small
  \centering
  % \vspace{-20mm}
  \includegraphics[width=\linewidth]{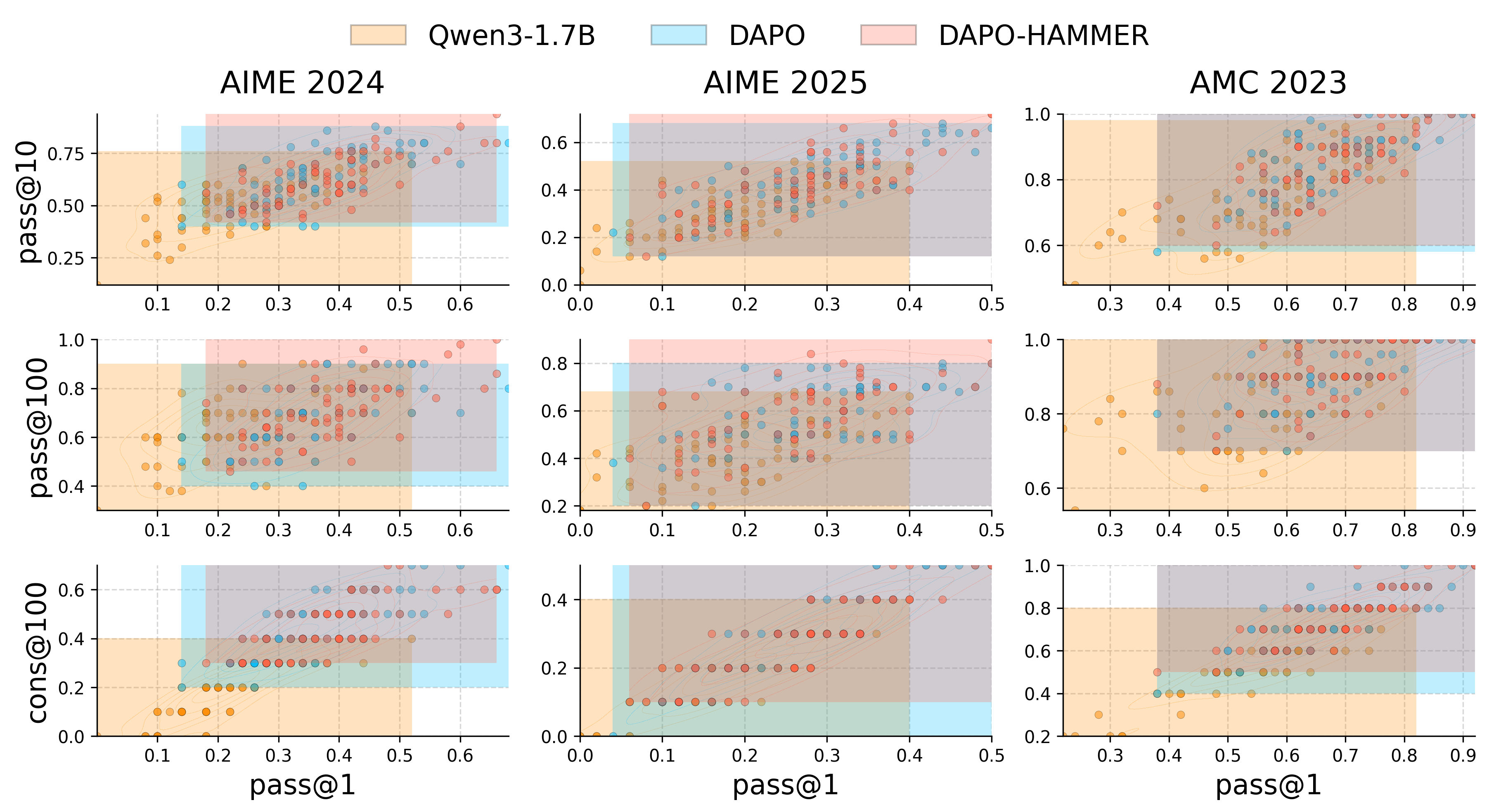}
  \vspace{-2mm}
    \caption{Distribution of metrics.}
  \vspace{-1mm}
  \label{fig: metric distribution}
\end{figure}

\vspace{-2mm}
\paragraph{Difficulty-based Training}
Many LLM curriculum RL approaches \cite{CL-E2H,qiu2025wisdom_E2H} use easy-to-hard (E2H) ordering. To compare, we train DAPO on AIME 2024 with different orders. As Figure~\ref{fig: e2h} shows, hard-to-easy (H2E) can match \textit{HAMMER} at its peak but later unstable, while E2H converges slower. \textit{HAMMER} achieves stable gains without costly difficulty estimation, requiring only a forward propagation of the backbone LLM (Section~\ref{sec: sentence embedding and similarity}).
% }

\vspace{-2mm}
\paragraph{Varying batch size}
% The training batch size significantly affects reinforcement learning. 
% Following recent RLVR guidelines \citep{gspo}, we set the train batch size equal to the mini-batch size and vary the train batch size to study \textit{HAMMER}'s effect. 
% As shown in Figure~\ref{fig: qwen3-1.7b dapo batchsize ablation}, increasing the batch size ($16, 32, 64$) generally improves performance. 
% \textit{HAMMER} benefits similarly from larger batch sizes and consistently outperforms the baselines. Larger batches expose the model to more diverse samples per iteration, promoting better generalization.
% \blue{
Training batch size is crucial in RLVR \citep{gspo}. We vary batch size (16,32,64) with train and mini-batch sizes set equal. As shown in Figure~\ref{fig: qwen3-1.7b dapo batchsize ablation}, larger batches improve performance. \textit{HAMMER} benefits similarly and consistently outperforms baselines by leveraging more diverse samples with bigger batch size.
% }

% \begin{figure}[htbp]
% % \small
%   \centering
%   \vspace{-10mm}
%   \includegraphics[width=\linewidth]{metric_distribution.png}
%   \vspace{-8mm}
%     \caption{Distribution of metrics.}
%   \vspace{-1mm}
%   \label{fig: metric distribution}
% \end{figure}

\vspace{-2mm}
\paragraph{Metric Distribution}
To examine how \textit{HAMMER} reshaping affects \textit{pass@1}, \textit{pass@10}, \textit{pass@100}, and \textit{cons@100}, we evaluated three versions of the Qwen3-1.7B model across AIME 2024, AIME 2025, and AMC 2023: the base model, the DAPO-enhanced model, and DAPO further augmented with \textit{HAMMER}.
As shown in Figure~\ref{fig: metric distribution}, at the same \textit{pass@1} level, \textit{HAMMER} consistently improves \textit{pass@k} for $k \ge 1$, shifting overall accuracy toward the upper-right. 

\vspace{-2mm}
\paragraph{Theory Validation}
\label{sec: theory validation}
To validate Theorem~\ref{thm: min DCS iff min similarity} and explaination in Example~\ref{ex: hammer advantages}, we compute $\mu_{\text{DCS}}$ and $\mu_{\text{NGM}}$ on DeepScaleR under varying subset ratios. 
As shown in Figure~\ref{fig: mu} show that \textit{HAMMER} prioritizes the most diverse samples with the same subset scale; 

% \brown{
% \clearpage
\section{Conclusion}
\vspace{-2mm}
We present \textit{HAMMER}, a novel schema that integrates semantic diversity into reinforcement for LLMs. 
By leveraging a minimum-semantic Hamiltonian path to define a curriculum sequence, \textit{HAMMER} stimulates early-stage model ``curiosity'', accelerates convergence, and improves training stability. 
Theoretically, we show that \textit{HAMMER} preserves the optimal policy while tightening generalization bounds through diverse sample, and that minimizing semantic similarity aligns with maximizing the dataset diversity measure $\mu_{\text{DCS}}$. 
Empirically, \textit{HAMMER} consistently enhances sample efficiency across multiple benchmarks, yielding 3\%--4\% average accuracy gains, demonstrating the effectiveness and generality of diversity-driven curriculum learning in LLM reinforcement training.
% }

% \section{Conclusion}
% \blue{
% We present \textit{HAMMER}, a curriculum RL schema that integrates semantic diversity for LLMs. 
% By constructing \textit{Hamiltonian Curiosity Order}, \textit{HAMMER} fosters early ``curiosity'', accelerates convergence, and improves stability. 
% We prove it preserves the optimal policy while tightening generalization bounds via diverse sample selection, with minimal similarity equivalent to maximizing dataset diversity $\mu_{\text{DCS}}$. 
% Empirically, \textit{HAMMER} achieves 3\%–4\% average accuracy gains across benchmarks, highlighting the effectiveness of \textit{HAMMER} in LLM reinforcement training.
% }

\newpage
\section*{Reproducibility Statement}
% \vspace{-1mm}
Algorithms, data and experimental details are included in the anonymous repository 
\url{https://anonymous.4open.science/r/HAMMER-B17F}
and provided as supplementary material. 
% \vspace{-3mm}
\section*{Ethics Statement}
% \vspace{-1mm}
All datasets used are publicly available with appropriate licenses. Our method is designed to improve LLM training efficiency, and should be used responsibly. We do not expect our work to produce harmful content, and encourage ethical deployment in line with the ICLR Code of Ethics.
% \vspace{-3mm}
\section*{The Use of Large Language Models (LLMs)}
% \vspace{-1mm}
Although the paper proposes a method to improve the training efficiency of LLMs. LLMs were used only to aid and polish the writing. No part of the research, method, or experiments relied on LLMs. The authors take full responsibility for the paper.

\bibliography{iclr2026_conference}
\bibliographystyle{iclr2026_conference}

\newpage
\appendix
\section{Experimental Setup}
\label{app: experimental setup}
\paragraph{Datasets}
All datasets are detailed in Table~\ref{tab: dataset details} in Appendix~\ref{app: dataset details}.
We evaluate our method on four benchmark datasets for mathematical problem solving. 
AIME 2024 (30 problems), AIME 2025 (30 problems),
AMC (83 problems) and
Olympiad (675 problems). 
All models are trained on the DeepScaleR (40,315 problems), which provides high-quality synthetic reasoning traces designed to enhance step-by-step mathematical reasoning. 
In \textit{HAMMER}, DeepScaleR is ordered by minimal similarity using Algorithm~\ref{alg: hamilton cycle}. 
The embedding of $\mathcal{X}$ is computed by mean pooling.
Our experiments show that varying $\eta$ has little impact on overall performance, so we fix $\eta=3$.
% all the time.

\paragraph{Baselines}
For the main experiments, we use DAPO and GRPO trained on randomly shuffled samples as baselines, 
while \textit{HAMMER} differs in the sample ordering. 
As shown in Table~\ref{tab: main results}, the baselines include Qwen3-1.7B-DAPO, Qwen3-1.7B-GRPO, Qwen3-4B-DAPO, and Qwen3-4B-GRPO.

\paragraph{Training} 
In our experiment, we adapt Qwen3-1.7B and Qwen3-4B \citep{qwen3_tech} as the backbone model, 
and train on \textit{verl} \citep{VeRL} through GRPO and DAPO.
Models for main experiment were trained with a batch size of 16 (including mini-batch size). The maximum prompt length is set to 1024 tokens, and the maximum response length is 8192 tokens. 
For the training hyper-parameters, learning rate is fixed at $1\times 10^{-6}$ without warmup step. For GRPO we adopt KL regularization (coefficient $\beta=0.001$). For DAPO, we set $\varepsilon_{\text{low}}=0.2$ and $\varepsilon_{\text{high}}=0.28$ and token-level policy gradient loss, and dynamically filter samples by accuracy during training. 
Each training step generates 16 rollouts, while validation (in dynamic experiments) uses 8 rollouts. The rollout temperature is 1.2, and the validation temperature is 0.6. 
For the reward, if the $i$-th rollout passes verification, it is assigned a positive reward $r_i = 1$; otherwise, it receives $r_i = 0$. 

\paragraph{Evaluation}
To evaluate LLM performance, we set temperature to $1.2$, with top-$p=0.95$ and top-$k=20$ with 8192 context length. For AIME 2024, AIME 2025, and AMC 2023, we sample $1$, $10$, and $100$ responses 10 times and report average $pass@k$ $(k \in {1,10,100})$ and $cons@100$, measuring solution accuracy and response consistency \citep{deepseek-r1}. For Olympiad, due to its larger size, we evaluate only $pass@1$, $pass@10$, $pass@32$, and $cons@32$, which are sufficient for reliable estimation.

% \subsection{Experiment Results}

\section{Proof of Section~\ref{sec: theoretical analysis}}
\label{app: proof of theorems}
\begin{lemma}[Vapnik-Chervonenkis Inequality]
\label{lem: vc inequality}
    For a policy class $\Pi$ with VC dimension $d$ and $n$ i.i.d. samples $\mathcal{S}$, the following inequality holds
    % \vspace{-1mm}
    $$
    \forall \varepsilon \in \mathbb{R}^+, \quad \mathbb{P}\left( \sup_{\pi \in \Pi} \left| \hat{\mathcal{R}}_{\mathcal{S}}(\pi) - \mathcal{R}(\pi) \right| \geq \varepsilon \right) \leq 2 \left( \frac{en}{d} \right)^d e^{-n\varepsilon^2/2}.
    $$
    Setting $\mathbb{P}$ to $\delta$ and solving for $\varepsilon$ yields the generalization bound.
% \vspace{-1mm}
\begin{equation*}
    % \label{eq: generation bound}
    \sup_{\pi \in \Pi} \left| \hat{\mathcal{R}}_{\mathcal{S}}(\pi) - \mathcal{R}(\pi) \right| \leq C\sqrt{\frac{d\log(n/d) + \log(1/\delta)}{n}}, \quad
\text{where $C > 0$ is some constant.}
\end{equation*}
\end{lemma}
\begin{proof}
The proof of Lemma~\ref{lem: vc inequality} relies on \textit{Hoeffding's inequality} \citep{proof_of_vc_inequality}.
Setting $\mathbb{P}$ to $\delta$ and solving for $\varepsilon$ yields the generalization bound (shorten $C\sqrt{\frac{d\log(n/d) + \log(1/\delta)}{n}}$ as $\rho$)
\end{proof}

\textbf{Theorem~\ref{thm: muti-stages training support condition}.}
\textit{
% \label{thm: muti-stages training support condition}
Given a subset $\mathcal{S} \subset \mathcal{X}$ of $n$ samples, let $\pi^*$ be the optimal policy on $\mathcal{X}$. There exists some $\gamma$ (i.e., $\gamma=2\rho$) such that
$
\pi^* \in \Pi_{\mathcal{S}}.
$
}

\begin{proof}
Let $\gamma = 2\rho$. From Inequality~\ref{eq: generation bound}, we have the uniform bound
% \vspace{-1mm}
\begin{equation}
\label{eq: uniform bound}
\forall \pi \in \Pi, \left| \hat{\mathcal{R}}_{\mathcal{S}}(\pi) - \mathcal{R}(\pi) \right| \leq \rho.
\end{equation}
Particularly, the optimal policy $\pi^*$ yields
% \vspace{-1mm}
\begin{equation}
\label{eq: uniform bound 1}
\hat{\mathcal{R}}_{\mathcal{S}}(\pi^*) \leq \mathcal{R}(\pi^*) + \rho.
\end{equation}

Moreover, applying the uniform bound to all policies and taking the minimum
% \vspace{-1mm}
\begin{equation}
\label{eq: uniform bound 2}
\hat{\mathcal{R}}_{\mathcal{S}}^* = \min_{\pi\in\Pi} \hat{\mathcal{R}}_{\mathcal{S}}(\pi) 
\geq 
\min_{\pi\in\Pi} [\mathcal{R}(\pi) - \rho ]
= \min_{\pi\in\Pi} \mathcal{R}(\pi) - \rho 
= \mathcal{R}(\pi^*) - \rho.
\end{equation}

Combining Inequality~\ref{eq: uniform bound 1} and \ref{eq: uniform bound 2}, we obtain
% \vspace{-1mm}
$$
\hat{\mathcal{R}}_{\mathcal{S}}(\pi^*) - \hat{\mathcal{R}}_{\mathcal{S}}^* 
\leq \left( \mathcal{R}(\pi^*) + \rho \right) - \left( \mathcal{R}(\pi^*) - \rho \right) 
= 2\rho.
$$

Thus, $\hat{\mathcal{R}}_{\mathcal{S}}(\pi^*) \leq \hat{\mathcal{R}}_{\mathcal{S}}^* + \gamma$, which by Definition~\ref{def: induced policy subset} implies $\pi^* \in \Pi_{\mathcal{S}}$.
\end{proof}

\textbf{Theorem~\ref{thm: local risk optimization}.}
\textit{
% \label{thm: local risk optimization}
For a subset $\mathcal{S}$ of $n$ samples, when $\gamma=2\rho$,
% \vspace{-1mm}
$$
\forall \pi \in \Pi_{\mathcal{S}}, \Delta_{\pi} \leq \mathcal{O}\left(\sqrt{\frac{d\log(n/d) + \log(1/\delta)}{n}}\right).
$$
}

\begin{proof}
    $\forall \pi \in \Pi_{\mathcal{S}}$, by Inequality~\ref{eq: uniform bound} and Definition~\ref{def: induced policy subset}, we have 
    \vspace{-1mm}
    $$
    \mathcal{R}(\pi) \leq \hat{\mathcal{R}}_{\mathcal{S}}(\pi) + \rho \leq 
    \left(
    \hat{\mathcal{R}}_{\mathcal{S}}^* + \gamma
    \right) + \rho = 
    (\hat{\mathcal{R}}(\pi^*) + \gamma) + \rho
    = \hat{\mathcal{R}}(\pi^*) + 3\rho
    $$
    \vspace{-1mm} 
    Thus, we deduce $\Delta_{\pi} = |\mathcal{R}(\pi)-\mathcal{R}(\pi^*)| \leq 3\rho = \mathcal{O}
    \left(
    \sqrt{\frac{d\log(n/d) + \log(1/\delta)}{n}}
    \right)
    $.
\end{proof}

\textbf{Theorem~\ref{thm: min DCS iff min similarity}.}
\textit{
Given a dataset $\mathcal{X} = \{x_i\}_{i=1}^n$ with embeddings $\{e_i\}_{i=1}^n$, let $\mathcal{S} \subset \mathcal{X}$ and $|\mathcal{S}|=m$, $M(\mathcal{S}) \in \mathbb{R}^{m\times m}$ be the semantic cosine similarity matrix of $\mathcal{S}$ with $M_{ij}(\mathcal{S}) = \langle e_i, e_j \rangle$, then we have
    $$
    \max_{\mathcal{S}\subset \mathcal{X}, |\mathcal{S}|=m} \mu_{\text{DCS}}(\mathcal{S}) \iff \min_{\mathcal{S}\subset \mathcal{X}, |\mathcal{S}|=m} \sum_{i=1}^m\sum_{j=1}^m \mathbb{I}(i\neq j)\cdot M_{ij}(\mathcal{S}).
    $$
}

\begin{proof}
    First, clarify the softmax convention: let $\operatorname{softmax}$ denote the row-wise softmax applied to matrix $M_{n \times n}$, i.e.
    $$
        \mathbb{P}_{ij}(\mathcal{S})
        \;=\;
        \frac{e^{M_{ij}(\mathcal{S})}}{\sum_{k=1}^m e^{M_{ik}(\mathcal{S})}}
        \qquad\text{for } i,j=1,\dots,m.
    $$
    By definition $\mu_{\text{DCS}}(\mathcal{S})=\textbf{tr}(\mathbb{P}(\mathcal{S}))=\sum_{i=1}^m \mathbb{P}_{ii}(\mathcal{S})$.
    Using the fact that each row of $\mathbb{P}(\mathcal{S})$ sums to $1$, we have for any fixed $\mathcal{S}$
    $
            \mathbb{P}_{ii}(\mathcal{S}) \;=\; 1 - \sum_{j\neq i} \mathbb{P}_{ij}(\mathcal{S}),
    $
    and therefore
    \[
        \mu_{\text{DCS}}(\mathcal{S})
        \;=\;
        \sum_{i=1}^m \mathbb{P}_{ii}(\mathcal{S})
        \;=\;
        m - \sum_{i=1}^m\sum_{j\neq i}\mathbb{P}_{ij}(\mathcal{S}).
    \]
    Hence maximizing $\mu_{\text{DCS}}(\mathcal{S})$ is equivalent to minimizing the total off-diagonal mass of $\mathbb{P}(\mathcal{S})$:
    \[
        \max_{\mathcal{S}\subset\mathcal{X},\,|\mathcal{S}|=m}\mu_{\text{DCS}}(\mathcal{S})
        \iff
        \min_{\mathcal{S}\subset\mathcal{X},\,|\mathcal{S}|=m}\sum_{i=1}^m\sum_{j\neq i}\mathbb{P}_{ij}(\mathcal{S}).
    \]

    Next, relate the off-diagonal entries $\mathbb{P}_{ij}(\mathcal{S})$ to the original similarity values $M_{ij}(\mathcal{S})$.
    For each fixed row $i$, $\mathbb{P}_{ij}(\mathcal{S})$ is a strictly increasing function of $M_{ij}(\mathcal{S})$ (holding the other entries in the same row fixed).
    In particular, increasing any off-diagonal similarity $M_{ij}$ (with other row-$i$ entries unchanged) strictly increases the corresponding $\mathbb{P}_{ij}$ and thus increases the row's off-diagonal mass $\sum_{j\ne i}\mathbb{P}_{ij}$.
    Consequently, a subset $\mathcal{S}$ that yields smaller off-diagonal similarity values $M_{ij}$  will also yield smaller total off-diagonal mass $\sum_{i\ne j}\mathbb{P}_{ij}$, and hence larger $\mu_{\text{DCS}}(\mathcal{S})$.

    Combining the two steps above, we obtain the stated equivalence at the level of optimization over subsets:
    \[
        \max_{\mathcal{S}\subset \mathcal{X},\, |\mathcal{S}|=m} \mu_{\text{DCS}}(\mathcal{S})
        \quad\Longleftrightarrow\quad
        \min_{\mathcal{S}\subset \mathcal{X},\, |\mathcal{S}|=m} \sum_{i=1}^m\sum_{j=1}^m \mathbb{I}(i\neq j)\cdot M_{ij}(\mathcal{S}).
    \]
\end{proof}

\newpage
\section{Supplementary Tables and Figures}
\begin{table}[h]
% \small
\centering
\caption{Comparison of $\mu_{\text{DCS}}$ values for different prefix subsets under two orders.}
\label{tab:mu_dcs_comparison}
\begin{tabular}{c ccccc}
\toprule
\textbf{Prefix Subset} & $\{x_1\}$ & $\{x_1,x_2\}$ & $\{x_{i}\}_{i=1}^3$ & $\{x_{i}\}_{i=1}^4$ & $\{x_{i}\}_{i=1}^5$ \\
\midrule
$\mu_{\text{DCS}}$ of $\mathcal{P}^*$ & 1.00 & 3.43 & 5.59 & 6.78 & 8.35 \\
$\mu_{\text{DCS}}$ of $\mathcal{P}_1$ & 1.00 & 2.49 & 3.96 & 5.24 & 8.35 \\
\bottomrule
\end{tabular}
\end{table}

% \begin{table*}[h]
% \small
% \centering
% \caption{Zero-shot performance on different backbone models (\(k=100\) for AIME 2024, AIME 2025, and AMC 2023; \(k=32\) for Olympiad).}
% \label{tab: zero-shot performance}
% \renewcommand{\arraystretch}{1.1}
% \setlength{\tabcolsep}{3pt}
% \begin{tabular}{l c c c c c c c c c c c}
% \toprule
% \multirow{2}{*}{\textbf{Dataset}}
% & \multicolumn{5}{c}{\textbf{Qwen3-1.7B}} & \multicolumn{5}{c}{\textbf{Qwen3-4B}} \\ 
% \cmidrule(r){2-6} 
% \cmidrule(r){7-11}
%  & pass@1 & pass@10 & pass@k & cons@k & Avg. & pass@1 & pass@10 & pass@k & cons@k & Avg. & \\ 
% \midrule
% AIME 2024 & 15.7 & 39.3 & 58.7 & 20.0 & 33.4 & 26.3 & 42.7 & 55.3 & 40.0 & 41.1 \\
% AIME 2025 & 18.7 & 26.7 & 28.7 & 35.3 & 23.3 & 17.0 & 28.7 & 35.3 & 23.3 & 26.1 \\
% AMC 2023  & 47.7 & 62.5 & 72.4 & 50.6 & 58.3 & 52.8 & 67.6 & 76.3 & 60.2 & 64.2 \\
% Olympiad & 42.1 & 53.3 & 55.7 & 43.6 & 48.6 & 45.5 & 55.3 & 58.3 & 46.5 & 51.4 \\
% \bottomrule
% \end{tabular}
% \end{table*}

\begin{figure}[h]
  \centering
  % \vspace{-6mm}
    \includegraphics[width=0.6\linewidth]{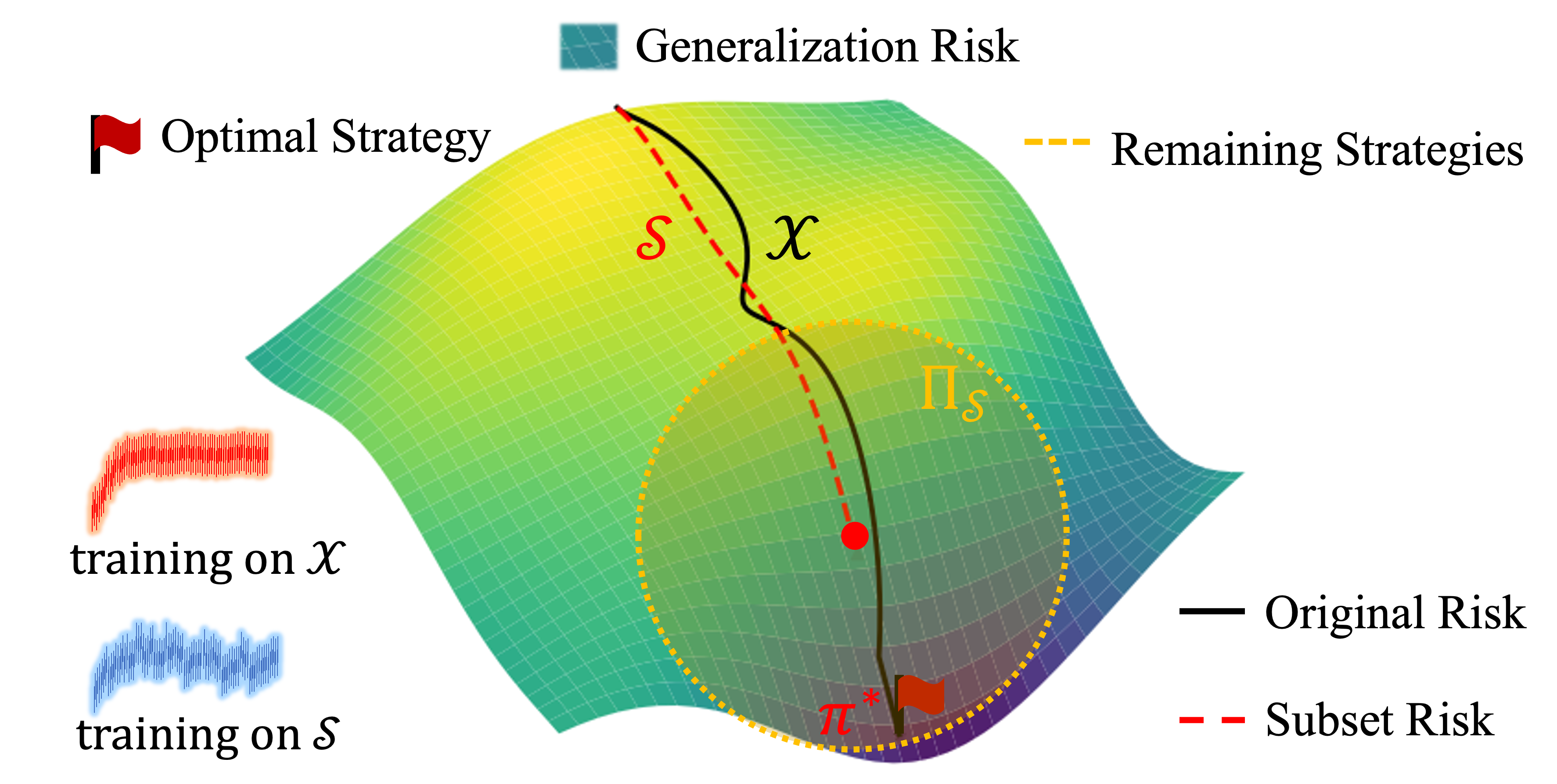}
\caption{Illustration of the benefits of introducing diversity early in training.}
  \label{fig: example of diversified sample training}
  % \vspace{-3mm}
\end{figure}

% \begin{figure}[htbp]
% % \vspace{-4mm}
% \centering
%   \subfigure[Maximum Hamiltonian similarity ablation study.]{
%     \includegraphics[width=0.48\linewidth]{qwen3-1.7b-dapo-with-mh.png} 
%     \label{fig: max hamilton}
%   }
%   \hfill
%   \subfigure[Difficulty-based ablation study.]{
%     \includegraphics[width=0.48\linewidth]{qwen3-1.7b-dapo-E2H.png}
%     \label{fig: e2h}
%   }
% % \vspace{-4mm}
% \caption{Data order ablation study, where DAPO-MAX denote \textit{max semantic similarity} data order, 
% DAPO-E2H and DAPO-H2E denote ``easy-to-hard'' and ``hard-to-easy'' data order.
% }
% \label{fig: max hamilton and e2h and batchsize}
% \end{figure}

\begin{figure}[H]
% \vspace{-3mm}
  \subfigure[$\mu_{\text{DCS}}$ on DeepScaleR varying $|\mathcal{S}|/|\mathcal{X}|$.] {
    \centering
    \includegraphics[width=0.45\linewidth]{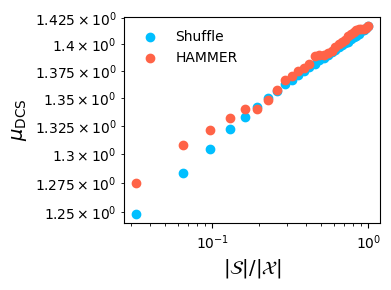} 
    \label{fig: mu_DCS}
  }
  % \hfill
  % \hspace{2mm}
  \subfigure[
      $\mu_{\text{NGM}}$ on DeepScaleR varying $|\mathcal{S}|/|\mathcal{X}|$.
    ]{
    \centering
    \includegraphics[width=0.43\linewidth]{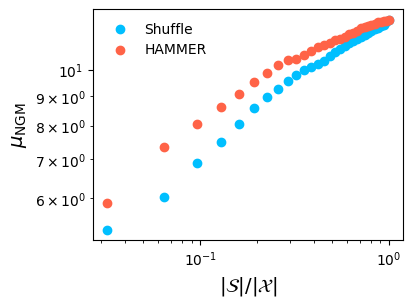}
    \label{fig: mu_NGM}
  }
  % \vspace{-3mm}
  \caption{
  $\mu$ on DeepScaleR with different subset ratios $|\mathcal{S}|/|\mathcal{X}|$.
  }
  \label{fig: mu}
\end{figure}

\newpage
\section{Related Works}
\paragraph{RL for LLM Reasoning}
Reinforcement Learning (RL) plays a critical role in LLM post-training. Traditional methods such as Reinforcement Learning from Human Feedback (RLHF) train a reward model to compare response preferences and optimize policy via algorithms like Proximal Policy Optimization (PPO) \citep{ppo}. To simplify, Direct Preference Optimization (DPO) \citep{RafailovSMMEF23} was proposed, which directly optimizes policy using pairwise preference data without reward model. A key development is Reinforcement Learning with Verifiable Rewards (RLVR), which offer feedback based on outcome correctness or verifiable reward, significantly enhancing LLMs' reasoning in mathematics and programming \citep{Survey_RLLM}. OpenAI’s o1 \citep{OpenAI-o1} advanced reasoning, while DeepSeek-R1 \citep{deepseek-r1}introduced zero-RL by eliciting model's slow-thinking capability. These advances spurred Large Reasoning Models (LRMs) like Kimi 1.5 \citep{kimi-1.5} and QwQ \citep{qwq32b}. A key RLVR algorithm, Group Relative Policy Optimization (GRPO), extends PPO by sampling multiple responses to compute group-relative advantage, yielding major gains \citep{deepseekmath_grpo}. It inspired variants like DAPO \citep{dapo} , VAPO \citep{Zhang_learning2reason_under_off_policy-vapo} and GSPO \citep{gspo}.

\paragraph{Curriculum RL}
Curriculum Learning (CL) draws inspiration from human education by structuring learning from simple to complex concepts \citep{CL_Survey}. Recent studies have explored curriculum-based reinforcement learning to improve reasoning and generalization in large language models \citep{bae2025onlinedifficultyfilteringreasoning,zeng2025simplerlzooinvestigatingtamingzero}. Some approaches assign difficulty levels to Chain-of-Thought (CoT) annotations \citep{qiu2025wisdom_E2H,CL-E2H}, filter out over-simple or over-challenging examples, or maintain a balanced distribution of task difficulties. Other methods use manually designed curricula that transition from easy to hard tasks after fixed training intervals \citep{Logic-RL,kimi-1.5}. 
However, setting courses according to the difficulty relies on the based model and requires evaluating based model's \textit{pass@k} performance. In this paper, rather than simply categorizing tasks by difficulty, we design the curriculum from diversity.

\paragraph{Coreset Selection}
Coreset selection (CS) accelerates training by selecting a compact, representative subset of samples \citep{_33_33_,_58_58_,_60_60_,_65_65_,_15_15_,_38_38_,task_cs_ft}. While effective for pruning redundancy, CS inherently faces performance bottlenecks \citep{cs_bottleneck}. In contrast, our Hamiltonian ordering adopts a curriculum learning view: it leverages semantic diversity to structure training, exposing the model to varied samples early on to promote more curious and stable learning. Unlike CS’s focus on reduction, our approach complements it by emphasizing diversity-driven guidance.

\paragraph{Data Diversity}
The evaluation of text diversity can be categorized into three approaches. 
(1) $N$-gram based methods \citep{n-gram_based_methods} utilize lexical statistics through metrics like distinct-$n$ \citep{distinct-n}, self-BLEU \citep{self-BLEU}, and ROUGE-L \citep{wang-etal-2023-self-instruct,padmakumar2024doeswritinglanguagemodels} to efficiently quantify surface-level variation. (2) Reference-based methods \citep{GANs} such as MAUVE \citep{MAUVE} quantify diversity by measuring the distributional divergence between generated texts and a high-quality reference dataset;
(3) Transformation-based methods \citep{transformer-based} employ learned representations (e.g., from language models) to capture multi-faceted diversity (semantic, syntactic, and stylistic) and summarize it via techniques like clustering \citep{Du_Black} or eigenvalue computation for VendiScore \citep{vendiscore} and its extensions RKE \citep{RKE} and FKEA \citep{FKEA}, which offer superior flexibility and comprehensiveness but suffer from higher computational complexity.

\newpage
\section{Dataset Details}
\label{app: dataset details}
\begin{table}[h]
\small
\centering
\renewcommand{\arraystretch}{1.2}
\begin{tabularx}{\textwidth}{l c X X}
\toprule
\textbf{Dataset} & \textbf{Size} & \textbf{URL} & \textbf{Description} \\
\midrule
\midrule
AIME 2024 & 30 & \url{https://huggingface.co/datasets/HuggingFaceH4/aime_2024} 
& The dataset consists of 30 problems from the 2024 AIME I\&II tests. \\
\midrule
AIME 2025 & 30 & \url{https://huggingface.co/datasets/opencompass/AIME2025} 
& This dataset contains problems from the American Invitational Mathematics Examination (AIME) 2025-I\&II. \\
\midrule
AMC 2023 & 83 & \url{https://huggingface.co/datasets/math-ai/amc23} 
& All 83 problems come from AMC 2023 competition. \\
\midrule
Olympiad & 675 & \url{https://huggingface.co/datasets/math-ai/olympiadbench} 
& A challenging benchmark for promoting AGI with olympiad-level problems. \\
\midrule
DeepScaleR & 40315 & \url{https://huggingface.co/datasets/agentica-org/DeepScaleR-Preview-Dataset} & The dataset consists of approximately 40000 unique mathematics problem-answer pairs compiled from:
AIME (American Invitational Mathematics Examination) problems (1984-2023) and
AMC (American Mathematics Competition) problems (prior to 2023).\\
\bottomrule
\end{tabularx}
\caption{Summary of datasets with URLs and descriptions. 
}
\label{tab: dataset details}
\end{table}

\end{document}